\documentclass[nohyperref]{article}

\usepackage{etoolbox}
\newtoggle{arxiv}
\toggletrue{arxiv}

\usepackage{microtype}
\usepackage{graphicx}
\usepackage{subfigure}
\usepackage{booktabs} %
\usepackage{tabularx}
\usepackage{hyperref}
\usepackage{multirow}
\usepackage{setspace}
\usepackage{xfrac}

\iftoggle{arxiv}{
  \usepackage[numbers]{natbib}
}
{
\usepackage{icml2022}
}

\usepackage{amsmath}
\usepackage{amssymb}
\usepackage{mathtools}
\usepackage{amsthm}
\usepackage{relsize}
\usepackage{comment}
\usepackage{algorithm}
\usepackage{algorithmic}

\usepackage{subfigure}

\usepackage[capitalize]{cleveref}

\usepackage[inline]{enumitem}

\usepackage[textsize=tiny]{todonotes}

\theoremstyle{plain}
\newtheorem{theorem}{Theorem}
\newtheorem{lemma}{Lemma}[section]
\newtheorem{proposition}[lemma]{Proposition}

\theoremstyle{definition}
\newtheorem{definition}[lemma]{Definition}
\newtheorem{assumption}[lemma]{Assumption}
\theoremstyle{remark}

\newcommand{\diag}{\mathrm{diag}}
\newcommand{\norm}[1]{\left\|{#1}\right\|} %

\newcommand{\B}{\mathcal{B}}
\newcommand{\M}{\mathcal{M}}

\newcommand{\BBS}{\B\B^*}

\newcommand{\MMS}{\M\M^*}

\newcommand{\BF}{\mathcal{B}\mathcal{F}}
\newcommand{\BD}{\mathcal{B}\mathcal{D}}
\newcommand{\DB}{\mathcal{D}\mathcal{B}}
\newcommand{\ind}[1]{^{(#1)}}

\newcommand{\vp}{\mathbf{p}}
\newcommand{\vq}{\mathbf{q}}

\newcommand{\vu}{\mathbf{u}}
\newcommand{\vv}{\mathbf{v}}

\newcommand{\vx}{\mathbf{x}}
\newcommand{\vy}{\mathbf{y}}
\newcommand{\vz}{\mathbf{z}}
\newcommand{\vA}{\mathbf{A}}
\newcommand{\vB}{\mathbf{B}}
\newcommand{\vC}{\mathbf{C}}
\newcommand{\vD}{\mathbf{D}}

\newcommand{\vF}{\mathbf{F}}
\newcommand{\vG}{\mathbf{G}}

\newcommand{\vI}{\mathbf{I}}

\newcommand{\vL}{\mathbf{L}}
\newcommand{\vM}{\mathbf{M}}

\newcommand{\vP}{\mathbf{P}}
\newcommand{\vQ}{\mathbf{Q}}
\newcommand{\vR}{\mathbf{R}}
\newcommand{\vS}{\mathbf{S}}

\newcommand{\R}{\mathbb{R}}
\newcommand{\F}{\mathbb{F}}

\newcommand{\ff}[2]{\tfrac{#1}{#2}}

\newcommand{\pd}[2]{\mathlarger{\prod\limits_{#1}^{#2}}}

\newcommand{\lt}{\left(}
\newcommand{\rt}{\right)}

\newcommand{\floor}[1]{\left\lfloor#1\right\rfloor}

\renewcommand{\mod}{\text{ mod }}

\newcommand{\MM}{\widetilde{\vM}}

\DeclareMathOperator*{\argmin}{argmin}

\newtoggle{comment}
\toggletrue{comment}
\togglefalse{comment} %

\iftoggle{comment}{
\newcommand{\Beidi}[1]{{\color{orange} [Beidi: {#1}]}}
\newcommand{\Tri}[1]{{\color{cyan} [Tri: {#1}]}}
\newcommand{\nimit}[1]{{\color{red} [Nimit: {#1}]}}
\newcommand{\micp}[1]{{\color{blue!70} [Michael: {#1}]}}
\newcommand{\arjun}[1]{{\color{green} [Arjun: {#1}]}}
}{
\newcommand{\Beidi}[1]{}
\newcommand{\Tri}[1]{}
\newcommand{\nimit}[1]{}
\newcommand{\micp}[1]{}
\newcommand{\arjun}[1]{}
}

\iftoggle{arxiv}{
  \setlength{\textwidth}{6.5in}
  \setlength{\textheight}{9in}
  \setlength{\oddsidemargin}{0in}
  \setlength{\evensidemargin}{0in}
  \setlength{\topmargin}{-0.5in}
  \newlength{\defbaselineskip}
  \setlength{\defbaselineskip}{\baselineskip}
  \setlength{\marginparwidth}{0.8in}
}{
\usepackage[compact]{titlesec}
\titlespacing{\section}{0pt}{*1.0}{*0}
\titlespacing{\subsection}{0pt}{*0}{*0}
\titlespacing{\subsubsection}{0pt}{*0}{*0}

\usepackage[subtle, mathdisplays=tight, charwidths=tight, leading=normal]{savetrees}

\def\setstretch#1{\renewcommand{\baselinestretch}{#1}}
\setstretch{0.99}
\addtolength{\parskip}{-0.3pt}
}

\iftoggle{arxiv}{
\title{Monarch: Expressive Structured Matrices for Efficient and Accurate Training}
\usepackage{authblk}
\author[1]{Tri Dao}
\author[1]{Beidi Chen}
\author[1]{Nimit Sohoni}
\author[1]{Arjun Desai}
\author[1]{Michael Poli}
\author[2]{Jessica Grogan}
\author[3]{Alexander Liu}
\author[3]{Aniruddh Rao}
\author[2]{Atri Rudra}
\author[1]{Christopher R{\'e}}
\affil[1]{Stanford University}
\affil[2]{University at Buffalo, SUNY}
\affil[2]{University of Michigan}
\affil[ ]{\texttt{\{trid,beidic,nims,arjundd,poli\}@stanford.edu}, \texttt{\{jrgrogan,atri\}@buffalo.edu}, \texttt{\{avliu,anrao\}@umich.edu}, \texttt{chrismre@cs.stanford.edu}}
}{
\icmltitlerunning{Monarch}
}

\begin{document}

\iftoggle{arxiv}{
  \maketitle
}{
\twocolumn[
\icmltitle{Monarch: Expressive Structured Matrices for Efficient and Accurate Training}

\icmlsetsymbol{equal}{*}

\begin{icmlauthorlist}
\icmlauthor{Firstname1 Lastname1}{equal,yyy}
\icmlauthor{Firstname2 Lastname2}{equal,yyy,comp}
\icmlauthor{Firstname3 Lastname3}{comp}
\icmlauthor{Firstname4 Lastname4}{sch}
\icmlauthor{Firstname5 Lastname5}{yyy}
\icmlauthor{Firstname6 Lastname6}{sch,yyy,comp}
\icmlauthor{Firstname7 Lastname7}{comp}
\icmlauthor{Firstname8 Lastname8}{sch}
\icmlauthor{Firstname8 Lastname8}{yyy,comp}
\end{icmlauthorlist}

\icmlaffiliation{yyy}{Department of XXX, University of YYY, Location, Country}
\icmlaffiliation{comp}{Company Name, Location, Country}
\icmlaffiliation{sch}{School of ZZZ, Institute of WWW, Location, Country}

\icmlcorrespondingauthor{Firstname1 Lastname1}{first1.last1@xxx.edu}
\icmlcorrespondingauthor{Firstname2 Lastname2}{first2.last2@www.uk}

\icmlkeywords{Machine Learning, ICML}

\vskip 0.3in
]

\printAffiliationsAndNotice{\icmlEqualContribution} %
}

\begin{abstract}
  Large neural networks excel in many domains, but they are expensive to train and fine-tune.
  A popular approach to reduce their compute/memory requirements is to replace dense weight matrices with structured ones (e.g., sparse, low-rank, Fourier transform).
  These methods have not seen widespread adoption (1) in end-to-end training due to
  unfavorable efficiency--quality tradeoffs, and
  (2) in dense-to-sparse fine-tuning due to lack of tractable algorithms to
  approximate a given dense weight matrix.
  To address these issues, we propose a class of matrices (Monarch) that is \emph{hardware-efficient} (they are parameterized as products of two block-diagonal matrices for better hardware utilization) and \emph{expressive} (they can represent many commonly used transforms).
  Surprisingly, the problem of approximating a dense weight matrix with a Monarch matrix, though nonconvex, has an analytical optimal solution.
  These properties of Monarch matrices unlock new ways to train and fine-tune sparse and dense models.
  We empirically validate that Monarch can achieve favorable accuracy–efficiency tradeoffs in several end-to-end sparse training applications: speeding up ViT and GPT-2 training on ImageNet classification and Wikitext-103 language modeling by 2$\times$ with comparable model quality, and reducing the error on PDE solving and MRI reconstruction tasks by 40\%.
  In sparse-to-dense training, with a simple technique called ``reverse sparsification,'' Monarch matrices serve as a useful intermediate representation to speed up GPT-2 pretraining on OpenWebText by 2$\times$ without quality drop.
  The same technique brings 23\% faster BERT pretraining than even the very optimized implementation from Nvidia that set the MLPerf 1.1 record.
  In dense-to-sparse fine-tuning, as a proof-of-concept, our Monarch approximation algorithm speeds up BERT fine-tuning on GLUE by 1.7$\times$ with comparable accuracy.
\end{abstract}

\section{Introduction}
\label{sec:intro}

Large neural networks excel in many domains, but their training and fine-tuning demand extensive computation and memory~\citep{kaplan2020scaling}.
A natural approach to mitigate this cost is to replace dense weight matrices with structured ones, such as sparse \& low-rank matrices and the Fourier transform.
However, structured matrices (which can be viewed as a general form of sparsity) have not yet seen wide adoption to date, due to two main challenges.
(1) In the \textbf{end-to-end} (E2E) training setting, they have shown unfavorable efficiency--quality tradeoffs.
Model \emph{efficiency} refers how efficient these structured matrices are on modern hardware (e.g., GPUs).
Model \emph{quality} (performance on tasks) is determined by how expressive they are (e.g., can they represent commonly used transforms such as convolution or Fourier/cosine transforms that encode domain-specific knowledge).
Existing structured matrices are either not hardware-efficient, or not expressive enough.
(2) In the setting of \textbf{dense-to-sparse} (D2S) fine-tuning of pretrained models, 
a long-standing problem for most classes of structured matrices is the lack of tractable algorithms to approximate dense pretrained weight matrices~\citep{pan2012structured}.

Sparse matrices have seen advances in training deep learning models (e.g., pruning~\citep{han2015deep}, lottery tickets~\citep{frankle2018lottery}), but most work on (entrywise) sparsification focuses on reducing training or inference FLOPs, which do not necessarily map to E2E training time on modern hardware (e.g., GPUs).
In fact, most sparse training methods \emph{slow down} training in wall-clock time~\citep{gale2019state, hooker2020hardware}.
Moreover, sparse matrices are not able to represent commonly used transforms such as convolution and the Fourier transform.
Another class of structured matrices, such as Fourier, sine/cosine, Chebyshev, are used in specialized domains such as PDE solving~\citep{trefethen2000spectral} and medical imaging~\citep{hsieh2003computed}.
However, they are difficult to use in E2E training since only specific instances of these structured matrices have fast GPU implementations (e.g., FFT). Moreover, their applications requires domain expertise to hand-pick the right transforms.
Generalizations of these transforms (e.g., Toeplitz-like~\citep{sindhwani2015structured}, orthogonal polynomial transforms~\citep{driscoll1997fast}, low-displacement rank~\citep{kailath1979displacement}, quasi-separable~\citep{eidelman1999new}), though learnable, often lack efficient implementation on GPUs~\citep{thomas2018learning} for E2E training as well.
In addition, they have no known tractable algorithm to approximate a given dense matrix~\citep{pan2012structured}, making them difficult to use in D2S fine-tuning.

\textbf{E2E training.}
The technical challenge in addressing the efficiency--quality tradeoff of structured matrices is to find a parameterization that is both efficient on block-oriented hardware (e.g., GPUs) and expressive (e.g., can represent many commonly used transforms).
We propose a class of matrices called Monarch,\footnote{They are named after the monarch butterfly.} parameterized as products of two block-diagonal matrices (up to permutation), to address this challenge.
This parameterization leverages optimized batch-matrix-multiply (BMM) routines on GPUs, yielding up to 2$\times$ speedup compared to dense matrix multiply (\cref{subsec:benchmark_tasks}).
We show that the class of Monarch matrices contains the class of butterfly matrices~\citep{parker1995random,dao2019learning}, which can represent any low-depth arithmetic circuits in near optimal runtime and parameter size~\citep{dao2020kaleidoscope}.
Monarch matrices inherit this expressiveness and thus can represent many fast transforms (e.g., Fourier, sine/cosine/Chebyshev transforms, convolution) (\cref{thm:Monarch_expressiveness}).
\begin{figure}[t]
  \centering
  \includegraphics[width=.49\textwidth]{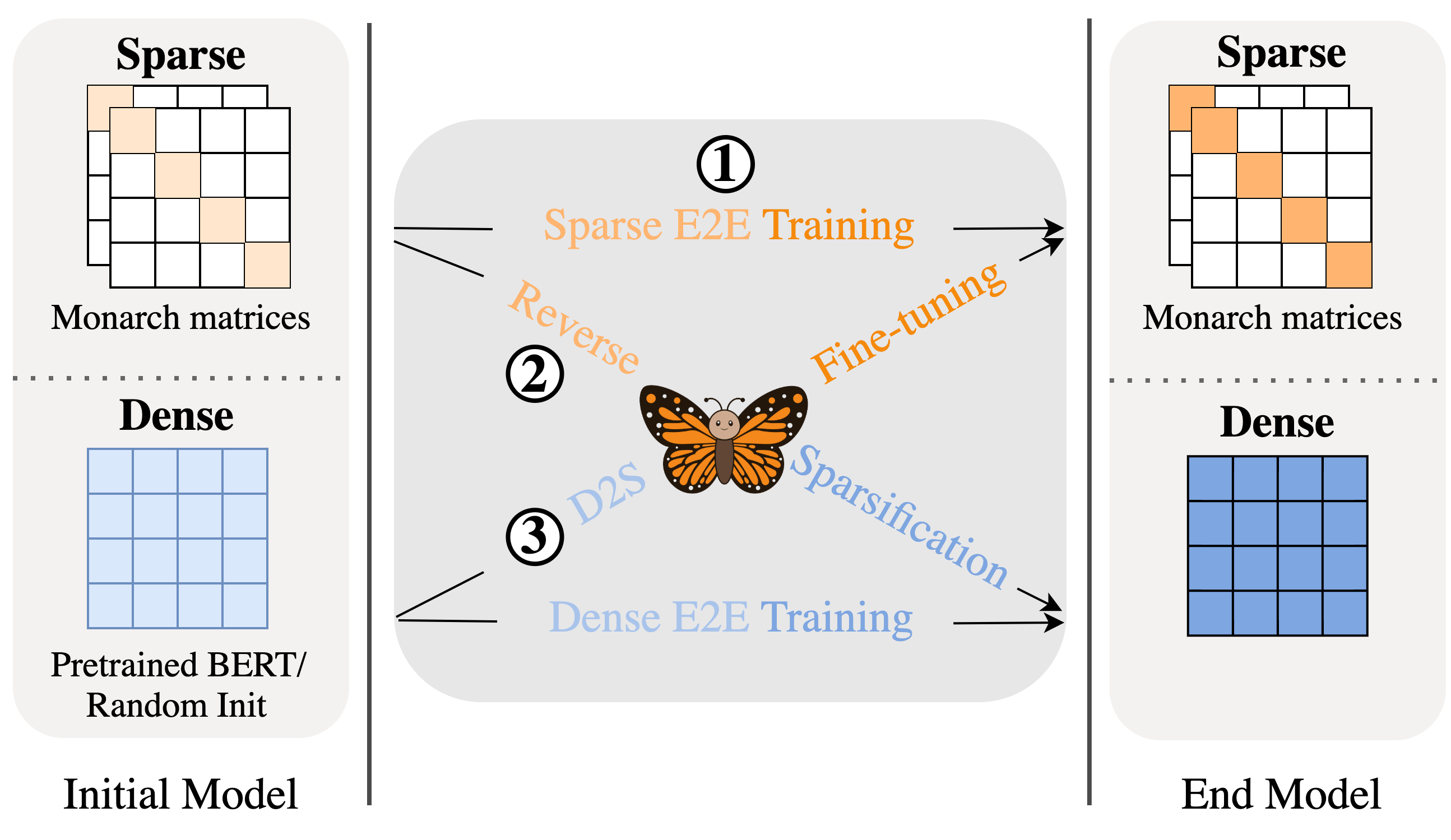}
  \label{fig:diagram}
  \vspace{-5mm}
  \caption{Monarch matrices unlock several ways to train sparse and dense models: end-to-end training a sparse (Monarch) model can be 2x faster than dense training thanks to its hardware efficiency; sparse-to-dense ``reverse sparsification'' can speed up training of large models such as GPT-2; and our dense-to-sparse Monarch projection algorithm can transfer knowledge from pretrained dense model to Monarch model and speed up BERT fine-tuning.}
  \vspace{-1.0em}
\end{figure}

\textbf{Sparse-to-dense (S2D) training, aka ``reverse sparsification''.} The hardware-efficiency and expressiveness of Monarch matrices unlock a new way to train dense models:
training with Monarch weight matrices for most of the time and then transitioning to dense weight matrices (\cref{fig:reverse_sparsification}).
This technique can be used in cases where sparse training faces representation or optimization difficulties~\citep{evci2019difficulty} or a dense model is necessary.
One such application is language modeling on large datasets, where a massive number of parameters are required~\citep{kaplan2020scaling} to memorize the textual patterns~\citep{geva2020transformer}.
Monarch matrices can serve as a fast intermediate representation to speed up the training process of the dense model.

\textbf{D2S fine-tuning.}
While transitioning from sparse to dense matrices is easy, the reverse direction is challenging.
The main technical difficulty is the \emph{projection} problem: finding a matrix in a class of structured matrices that is the closest to a given dense matrix.
Only a few specific classes of structured matrices have a tractable projection solution, such as entrywise sparse matrices (magnitude pruning~\citep{tewarson1973sparse}), low-rank matrices (the Eckart-Young theorem~\citep{eckart1936approximation}), and orthogonal matrices (the orthogonal Procrustes problem~\citep{schonemann1966generalized}).
For more expressive classes of structured matrices, projection remains a long-standing problem~\citep{pan2012structured}.
For example, \citet{desa2018two} show that all structured matrices (in the form of arithmetic circuits) can be written as products of sparse matrices, which can be represented as products of butterfly matrices~\citep{dao2020kaleidoscope}.
There have been numerous heuristics proposed to project on the set of butterfly matrices or products of sparse matrices, based on iterative first-order optimization~\citep{le2016flexible, dao2019learning, khalitov2021sparse} or alternating minimization~\citep{lin2021deformable}.
However, they lack theoretical guarantees.
In contrast, we derive a projection algorithm for our Monarch parameterization and prove that it finds the optimal solution (\cref{thm:Monarch_projection}).
We also derive an algorithm to factorize matrices that are products of Monarch matrices (\cref{subsec:recovery}). These new algorithms allows us to easily finetune a pretrained model into a model with Monarch weight matrices (\cref{subsec:finetuning}).

We validate our approach empirically in these three settings, showing that our Monarch matrix parameterization achieves a favorable efficiency--accuracy tradeoff compared to baselines on a wide range of domains: text, images, PDEs, MRI.
\iftoggle{arxiv}{}{
\vspace{-0.25em}
}
\begin{itemize}[leftmargin=*,nosep,nolistsep,noitemsep]
\item In the \textbf{E2E sparse training} setting (\cref{subsec:e2e_training}), our Monarch matrices model trains 2$\times$ faster than dense models while achieving the same accuracy / perplexity on benchmark tasks (ViT on ImageNet classification, GPT-2 on Wikitext-103 language modeling).
On scientific and medical tasks relying on hand-crafted fast transforms (PDE solving, MRI reconstruction), Monarch reduces the error by up to 40\% at the same training speed compared to domain-specific Fourier-based methods.
\item In the \textbf{S2D training} setting (\cref{subsec:s2d_training}), our ``reverse sparsification'' process with Monarch matrices speeds up GPT-2 pretraining on the large OpenWebText dataset by 2$\times$ compared to an optimized implementation from NVIDIA~\citep{shoeybi2019megatron}, with comparable upstream and downstream (text classification) quality.
When applied to BERT pretraining, our method is 23\% faster than the implementation from Nvidia that set the MLPerf~\citep{mattson2020mlperf} 1.1 record.
\item In the \textbf{D2S fine-tuning} setting (\cref{subsec:finetuning}), we show a proof of concept that our Monarch projection algorithm speeds up BERT fine-tuning.
We project a pretrained BERT model to a Monarch matrix model and fine-tune on GLUE, with 2$\times$ fewer parameters, 1.7$\times$ faster fine-tuning speed, and similar average GLUE accuracy as the dense model.\footnote{Monarch code is available at \url{https://github.com/HazyResearch/monarch}}
\end{itemize}

\section{Related Work and Background}
\label{sec:related_work}

\subsection{Related Work}

\textbf{Sparse Training.}
Sparse training is an active research topic. There has been inspiring work along the line of compressing models such as neural network pruning and lottery tickets~\citep{han2015deep,han2015learning, frankle2018lottery}. Pruning methods usually eliminate neurons and connections through iterative retraining~\citep{han2015deep,han2015learning,sanh2020movement} or at runtime~\citep{NIPS2017_a51fb975,dong2017learning}. 
Although both Monarch and pruning methods aim to produce sparse models, we differ in our emphasis on \emph{overall} efficiency, whereas pruning mostly focuses on inference efficiency and disregards the cost of finding the smaller model. Lottery tickets \citep{frankle2018lottery,frankle2019stabilizing,frankle2020linear} are a set of small sub-networks derived from a larger dense network, which outperforms their parent networks in convergence speed and potentially in generalization. Monarch can be roughly seen as a class of manually constructed lottery tickets.

\textbf{Structured Matrices.}
Structured matrices are those with subquadratic ($o(n^2)$ for dimension $n \times n$) number of parameters and runtime.
Examples include sparse and low-rank matrices, and fast transforms (Fourier, Chebyshev, sine/cosine, orthogonal polynomials).
They are commonly used to replace the dense weight matrices of deep learning models, thus reducing the number of parameters and training/inference FLOPs.
Large classes of structured matrices (e.g., Toeplitz-like~\citep{sindhwani2015structured}, low-displacement rank~\citep{kailath1979displacement}, quasi-separable~\citep{eidelman1999new}) have been shown to be able to represent many commonly used fast transforms.
For example, \citet{desa2018two} show that a simple divide-and-conquer scheme leads to a fast algorithm for a large class of structured matrices.
Our work builds on butterfly matrices~\citep{parker1995random, dao2019learning},
which have been shown to be expressive but remain hardware-inefficient.
Pixelated butterfly~\citep{chen2021pixelated} has attempted to make butterfly matrices more hardware-friendly, but at the cost of reduced expressiveness.
Furthermore, it is not known if one can directly decompose a dense pretrained model to a model with butterfly weight matrices without retraining.

\subsection{Butterfly Matrices}
\label{sec:butterfly}
Our work builds on recent work on \emph{butterfly matrices}. \citet{dao2019learning} introduced the notion of a {butterfly matrix} as a certain product of permuted block-diagonal matrices, inspired by the Cooley-Tukey fast Fourier transform algorithm~\citep{cooley1965algorithm}.
They encode the divide-and-conquer structure of many fast multiplication algorithms.
\citet{dao2020kaleidoscope} showed that all structured matrices can be written as products of such butterfly matrices, and this representation has optimal memory and runtime complexity up to polylogarithmic factors. We now review these definitions (following \citep{dao2020kaleidoscope}).

    A \textbf{butterfly factor} of size $k$ (where $k$ is even) is a matrix of the form
    \(
        \begin{bmatrix}
            \vD_1 & \vD_2 \\ \vD_3 & \vD_4
        \end{bmatrix}
    \)
    where each $\vD_i$ is a $\frac{k}{2} \times \frac{k}{2}$ diagonal matrix. We call this class of matrices $\BF\ind{k,k}$.

 A \textbf{butterfly factor matrix} of size $n$ and block size $k$ is a block diagonal matrix of $\frac{n}{k}$ butterfly factors of size $k$:
     \[
        \mathrm{diag}\left(\vB_1, \vB_2, \hdots, \vB_\frac{n}{k} \right),
     \]
      where $\vB_i \in \BF\ind{k,k}$. We call this class of matrices $\BF\ind{n,k}$.

    Finally, a \textbf{butterfly matrix} of size $n = 2^s$ is a matrix $\vM$ that can be expressed as a product of butterfly factor matrices:
    \[
        \vM = \vB_n \vB_{n/2} \hdots \vB_2,
    \]
    where each $\vB_i \in \BF\ind{n, i}$. We denote the set of size-$n$ butterfly matrices by $\B\ind{n}$.
    Equivalently, $\vM$ can be written in the following form:
    \[
         \vM = \vB_n \begin{bmatrix}
            \vM_1 & 0 \\
            0 & \vM_2
         \end{bmatrix},
    \]
    where $\vB_n \in \BF\ind{n,n}$ and $\vM_1, \vM_2 \in \B\ind{\frac{n}{2}}$.

    \citet{dao2020kaleidoscope} further introduce the \emph{kaleidoscope matrix hierarchy}: the class $\B\B^{*(n)}$ is the set of matrices of the form $\vM_1\vM_2^*$ for $\vM_1,\vM_2 \in \B\ind{n}$, and the class $(\B\B^{*(n)})^w_e$ is the set of all matrices of the form $\lt \prod\limits_{i=1}^{w} \vM_i \rt[1{:}n, 1{:}n]$ where each $\vM_i \in \B\B^{*(e\cdot n)}$. ($\vA^*$ denotes the conjugate transpose of $\vA$.)
    When the size $n$ is clear from context, we will omit the superscript $\ind{n}$ (i.e., just write $\B,\BBS$, etc.).
As shown by Theorem 1 of \citet{dao2020kaleidoscope}, the kaleidoscope hierarchy can represent any structured matrix with nearly-optimal parameters and runtime: if $\vM$ is an $n \times n$ matrix such that multiplying any vector $v$ by $\vM$ can be represented as a linear arithmetic circuit with depth $d$ and $s$ total gates, then $\vM \in (\B\B^{*(n)})^{O(d)}_{O(s/n)}$. \nimit{this is probably too much detail.}

\section{Monarch: Definition \& Algorithms}
\label{sec:theory}

In \cref{subsec:parametrization}, we introduce \emph{Monarch matrices},
and describe how they relate to butterfly matrices. In \cref{subsec:ee} we show that the class of Monarch matrices is at least as expressive as the class of butterfly matrices,
while admitting a practically efficient representation.
In particular, many fast transforms (e.g., Fourier, convolution) can be represented as a Monarch matrix or as the product of two or four Monarch matrices (\cref{thm:Monarch_expressiveness}).
In \cref{subsec:projection}, we show how to project onto the set of Monarch
matrices. This allows us to tractably approximate a given matrix
(e.g., a dense pretrained weight matrix) with a Monarch matrix, unlocking new applications~(cf. \cref{sec:experiments}).
In \cref{subsec:recovery}, we show how to recover the individual factors of the
larger class of products of two Monarch matrices.

\subsection{Monarch Parametrization for Square Matrices}
\label{subsec:parametrization}

Inspired by the 4-step FFT algorithm~\citep{bailey1990ffts}, we propose
the class of Monarch matrices, each 
parametrized as the product of two block-diagonal matrices up to permutation:
\begin{definition}\label{def:Monarch}
  Let $n = m^2$. An $n \times n$ \emph{Monarch matrix} has the form:
  \begin{equation*}
    \vM = \vP \vL \vP^\top \vR,
  \end{equation*}
  where $\vL$ and $\vR$ are block-diagonal matrices, each with $m$ blocks of
  size $m \times m$, and $\vP$ is the permutation that maps
  $[x_1, \dots, x_n]$ to
  $[x_1, x_{1+m}, \dots, x_{1+(m-1)m}, x_2, x_{2+m}, \dots, \newline x_{2+(m-1)m}, \dots, x_{m}, x_{2m}, \dots, x_n]$.
\end{definition}
We call this the \emph{Monarch parametrization}. We denote the class of all
matrices that can be written in this form as $\M\ind{n}$ (dropping the superscript when clear from context).
\cref{fig:blockdiag_parametrization} illustrates this parametrization.
\begin{figure}[t]
  \centering
  \includegraphics[width=.45\textwidth]{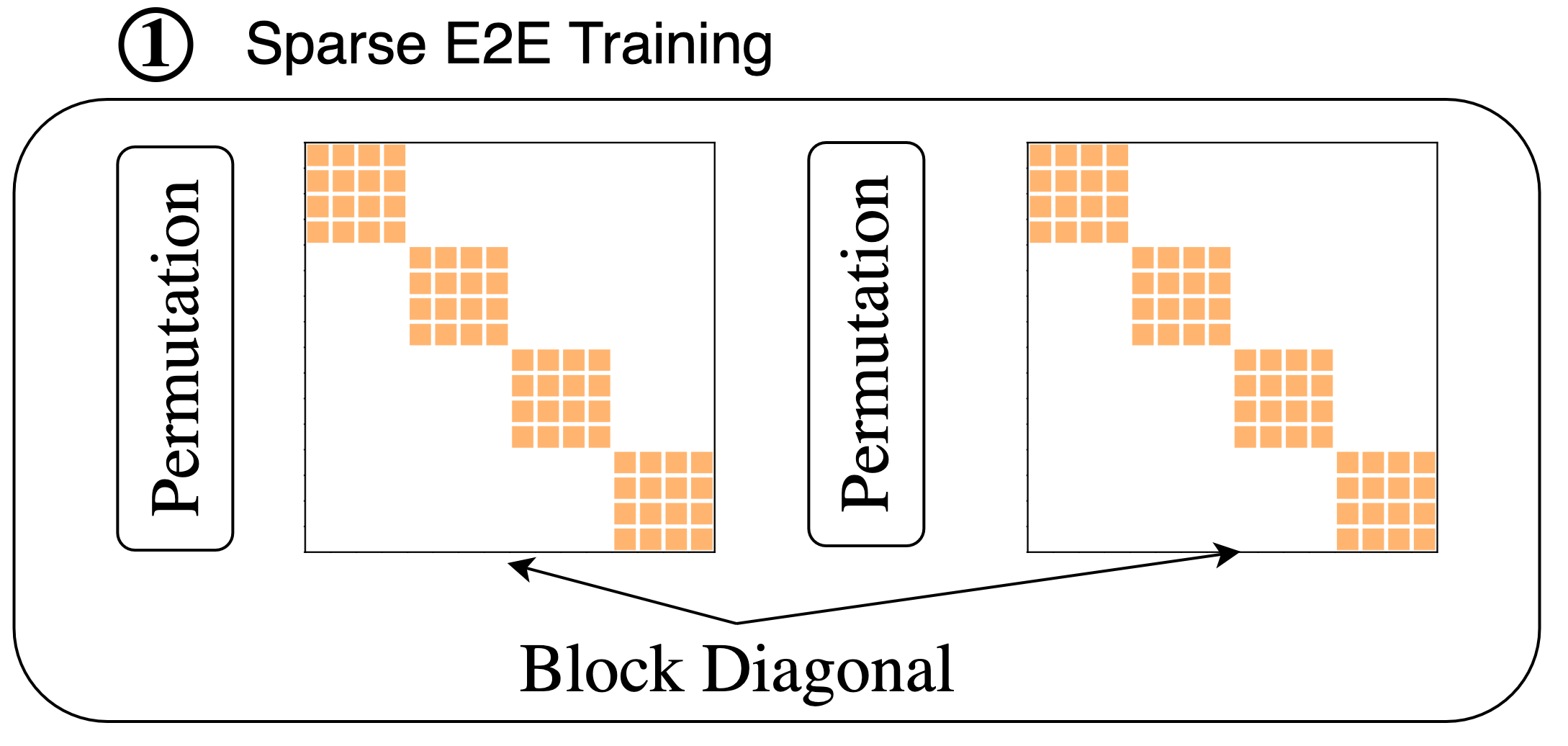}
  \vspace{-1em}
  \caption{\label{fig:blockdiag_parametrization}Monarch matrices are parametrized as products of two block-diagonal
    matrices up to permutation, allowing efficient multiplication algorithm that leverages batch
    matrix multiply.}
  \vspace{-0.5em}
\end{figure}

We now provide more intuition for this parametrization and connect it to butterfly
matrices.
For ease of exposition, suppose $\vB \in \B^{(n)}$ where $n$ is a power of 4.
Then let $\vL'$ be obtained by multiplying together the first $\ff{\log_2 n}{2}$
butterfly factor matrices in the butterfly factorization of $\vB$,
and $\vR$ by multiplying together the last $\ff{\log_2 n}{2}$ butterfly factor matrices.
(We detail this more rigorously in \cref{thm:b_contained}.)

The matrix $\vR$ is block-diagonal with $m = \sqrt{n}$ dense blocks, each block of size
$m \times m$:
$
  \vR = \diag(\vR_1, \dots, \vR_{m}).
$

The matrix $\vL'$ is composed of $m \times m$ blocks of size
$m \times m$, where each block is a diagonal matrix:
\begin{equation*}
  \vL' =
  \begin{bmatrix}
    \vD_{11} & \hdots & \vD_{1m} \\
    \vdots & \ddots & \vdots \\
    \vD_{m1} & \hdots & \vD_{mm} \\
  \end{bmatrix}.
\end{equation*}

The matrix $\vL'$ can also be written as block-diagonal with the same structure as $\vR$
after permuting the rows and columns.
Specifically, let $\vP$ be the permutation of Definition~\ref{def:Monarch}.
We can interpret $\vP$ as follows: it reshapes the vector $x$ of size $n$ as a matrix of size
$m \times m$, transposes the matrix, then converts back into a
vector of size $n$. Note that $\vP = \vP^\top$.
Then we can write
\begin{equation*}
  \vL = \vP \vL' \vP^\top, \quad \text{where } \vL = \diag(\vL_1, \dots, \vL_{m}).
\end{equation*}
Hence, up to permuting rows and columns, $\vL'$ is also a block-diagonal matrix of
$m$ dense blocks, each of size $m \times m$.

Thus we can write $\vB = \vP \vL \vP^\top \vR$,
where $\vL$, $\vR$, and $\vP$ are as in \cref{def:Monarch}.
So, $\vB \in \B\ind{n}$ implies that $\vB \in \M\ind{n}$.

\textbf{Products of Monarch Matrices.}
Another important class of matrices (due to their expressiveness, cf.\ \cref{thm:Monarch_expressiveness})
is the class $\MMS$: matrices that can be written as $\vM_1\vM_2^*$ for some $\vM_1, \vM_2 \in \M$. Further, $(\MMS)^2$ denotes the class of matrices that can be written $\vM_1\vM_2$ for $\vM_1, \vM_2 \in \MMS$.

\textbf{Extension to Rectangular Matrices.}
In practice, we also want a way to parametrize rectangular weight matrices, and to
increase the number of parameters of Monarch matrices to fit different
applications (analogous to the rank parameter in low-rank matrices and the
number of nonzeros in sparse matrices).
We make the simple choice to increase the block size of the block-diagonal
matrices in the Monarch parametrization, and to allow rectangular blocks. More details are in \cref{sec:permutation}.

\subsection{Expressiveness and Efficiency}
\label{subsec:ee}
We remark on the expressiveness of Monarch matrices and their products (ability
to represent many structured transforms), and on their computational and memory efficiency. 
\subsubsection{Expressiveness}
As described in Section \ref{subsec:parametrization}, any matrix $\vB \in \B\ind{n}$ can be written in the Monarch butterfly
representation, by simply condensing the $\log_2 n$ total factors into two matrices.
Thus, the Monarch butterfly representation is strictly more general than the original
butterfly representation (as there also exist matrices in $\M\ind{n}$ but not $\B\ind{n}$).
In other words, for a given size $n$, $\M \supset \B$; similarly $\MMS \supset \BBS$. 
In particular, \citet{dao2020kaleidoscope} showed that the following matrix classes are contained in $\BBS$,
which implies they are in $\MMS$ as well:
\begin{proposition}\label{thm:Monarch_expressiveness}
  The matrix class $\MMS$ can represent convolution, Hadamard
  transform, Toeplitz matrices~\citep{gray2006toeplitz}, and AFDF
  matrices~\citep{moczulski2015acdc}.
  The matrix class $(\MMS)^2$ can represent the Fourier transform, discrete sine and cosine
  transforms (DST/DCT), the $(HD)^3$~\citep{yu2016orthogonal} class,
  Fastfood~\citep{le2013fastfood}, and ACDC matrices~\citep{moczulski2015acdc}.
\end{proposition}

\subsubsection{Efficiency}
\textbf{Parameters.} A Monarch matrix $\vM = \vP\vL\vP^\top \vR$ is described by $2 n \sqrt{n}$ parameters:
$\vL, \vR$ both have $\sqrt{n}$ dense blocks of size $\sqrt{n} \times \sqrt{n}$, for a total
parameter count of $n\sqrt{n}$ each. The permutation $\vP$ is \emph{fixed}, and thus doesn't add any parameters.
\textbf{Speed.}
To multiply by $\vM$, we need to multiply by a block diagonal matrix $\vR$, permute,
multiply by a block diagonal matrix $\vL$, and finally permute.
All four of these steps can be implemented efficiently.
The total number of FLOPs is $O(n \sqrt{n})$, which is more the $O(n \log n)$ for
a butterfly matrix.
However, since we can leverage efficient block-diagonal multiplication (e.g.,
batch matrix multiply), Monarch multiplication is easy to
implement and is fast in practice (2x faster than dense multiply, cf.\ \cref{sec:experiments}).

\subsection{Projection on the Set $\M$ of Monarch Matrices}
\label{subsec:projection}

Given our class of structured matrices, a natural question is the
\emph{projection} problem: finding a Monarch matrix that is the closest to a
given dense matrix.
We show that this problem has an analytical optimal solution, and show how to compute it efficiently.
This allows us to project dense models to Monarch models, enabling D2S
fine-tuning (\cref{subsec:finetuning}).

We formalize the problem: for a given matrix $\vA$, find
\begin{equation}
  \label{eq:projection_objective}
  \argmin\limits_{\vM \in \mathcal{M}} \norm{\vA - \vM}^2_F.
\end{equation}

Even though this problem is nonconvex (as $\vM$ is parametrized as the product of
two matrices), in Theorem~\ref{thm:Monarch_projection} we
show that there exists an analytical solution (full proof in~\cref{sec:proofs}).
This is analogous to the Eckart-Young theorem that establishes that
optimal low-rank approximation is obtained from the SVD~\citep{eckart1936approximation}.
\begin{theorem}\label{thm:Monarch_projection}
  Given an $n \times n$ matrix $\vA$, there is an $O(n^{5/2})$-time algorithm
  that optimally solves the projection problem~\eqref{eq:projection_objective},
  and returns the Monarch factors $\vL$ and $\vR$.
\end{theorem}

We now derive this algorithm (\cref{alg:project}) by examining the structure of
a Monarch matrix $\vM$.

We first rewrite the steps of Monarch matrix-vector multiplication (i.e., computing $\vM \vx$).
The main idea is to view the input $\vx$, which is a vector of size $n = m^2$, as a 2D
tensor of size $m \times m$.
Then the two matrices $\vL$ and $\vR$ in the Monarch parametrization $\vM = \vP\vL\vP^\top \vR$ correspond
to batched matrix multiply along one dimension of $\vx$, followed by batched matrix
multiply along the other dimension of $\vx$.
Thus we view $\vx$ as a 2D tensor of size $m \times m$, and each of
$\vL$ and $\vR$ as a 3D tensor of size $m \times m \times m$.

Steps to multiply $\vx$ by a Monarch matrix $\vM = \vP \vL \vP^\top \vR$:
\begin{enumerate}[leftmargin=*,nosep,nolistsep,noitemsep]
  \item Multiply $\vR$ by $\vx$: $y_{kj} = \sum_i R_{kji} x_{ki}$, to obtain an output $\vy$ that is a
  2D tensor of size $m \times m$.
  \item Multiply $\vP \vL \vP^\top$ by $\vy$: $z_{\ell j} = \sum_k L_{j\ell k} y_{kj}$, to obtain an output
  that is a 2D tensor of size $m \times m$.
  \item Reshape $\vz$ back into a vector of size $n$, and return this.
\end{enumerate}
We can thus write the output $\vz$ as $z_{\ell j} = \sum_{k, i} L_{j\ell k} R_{kji} x_{ki}$.
\newline Since $\vM = \vP \vL \vP^\top\vR$, we can write:
\begin{equation}
  \label{eq:b_einsum}
  M_{\ell jki} = L_{j\ell k} R_{kji}.
\end{equation}
Note that here we view $\vM$ as a 4D tensor of size
$m \times m \times m \times m$.

When viewed as a 4D tensor, the structure of the matrix $\vM$ becomes
apparent, and the solution to the projection problem is easy to see.
Let's examine \cref{eq:b_einsum}: $M_{\ell jki} = L_{j\ell k} R_{kji}$.
We see that this reshaped tensor version of $\vM$ is simply $m \cdot m$ batches of
rank-1 matrices: we batch over the dimensions $k$ and $j$, and each batch is
simply a rank-1 matrix $(\vp_{jk}) (\vq_{jk})^\top$ for some length-$m$ vectors $\vp_{jk}, \vq_{jk}$.

Therefore, the projection objective (\cref{eq:projection_objective}) can be broken up
into the sum of $m \cdot m$ independent terms, each term corresponding to a
block of $\vA$ of size $m \times m$.
As the structure of a Monarch matrix forces each
block to have rank 1 as described above,
the solution to the projection problem becomes apparent:
given a matrix $\vA$, reshape it to a 4D tensor of size
$m \times m \times m \times m$, and take the rank-1
approximation of each batch with the SVD, which (after reshaping)
yields the factors $\vL, \vR$ of the desired matrix $\vM \in \mathcal{M}$.
(Note that if $\vA \in \M$ itself, this algorithm recovers the factors such that $\vA = \vP \vL \vP^\top\vR$.)

\begin{algorithm}[H]
  \caption{\label{alg:project}Projection on the set of Monarch matrices}
  \begin{algorithmic}
    \REQUIRE Matrix $\vA \in \R^{n \times n}$, with $n = m^2$.
    \STATE Reshape $\vA$ into a 4D tensor $\widetilde{\vA}$ of size $m \times m \times m \times m$, where
    $\widetilde{\vA}_{\ell jki} = \vA_{(\ell - 1) m + j, (k-1)m + i}$ for $\ell, j, k, i = 1, \dots, m$.\\
    \FOR{$1 \le j, k \le m$}
    \STATE Let $\widetilde{\vM}_{jk} = \widetilde{\vA}_{:, j, k, :}$ of size $m \times m$.
    \STATE Compute the best rank-1 approximation of $\widetilde{\vM}_{jk}$ as $\vu_{jk} \vv_{jk}^\top$ with the SVD of $\widetilde{\vA}$.
    \ENDFOR
    \STATE Let $\widetilde{\vR}$ be the $m \times m \times m$ tensor where $\widetilde{\vR}_{kji} = (\vv_{jk})_i$.
    \STATE Let $\widetilde{\vL}$ be the $m \times m \times m$ tensor where $\widetilde{\vL}_{j \ell k} = (\vu_{jk})_\ell$.
    \STATE Return $\widetilde{\vL}$, $\widetilde{\vR}$ as block-diagonal matrices $\vL, \vR$ (where the $b^{th}$ block of $\vL,\vR$ are $\widetilde{\vL}_{b,:,:}$, $\widetilde{\vR}_{b,:,:}$ respectively)
  \end{algorithmic}
\end{algorithm}

\subsection{Factorization of $\MMS$ Matrices}
\label{subsec:recovery}
In the previous section, we saw how to project onto the set $\M$.
As Theorem~\ref{thm:Monarch_expressiveness} shows, the broader class $\MMS$ also encompasses many important linear transforms.
In this section, we present an algorithm to compute the Monarch factorization of a given matrix $\vM \in \MMS$, under mild assumptions.
This allows us to store and apply $\vM$ efficiently.

\begin{figure}[t]
  \centering
  \includegraphics[width=.4\textwidth]{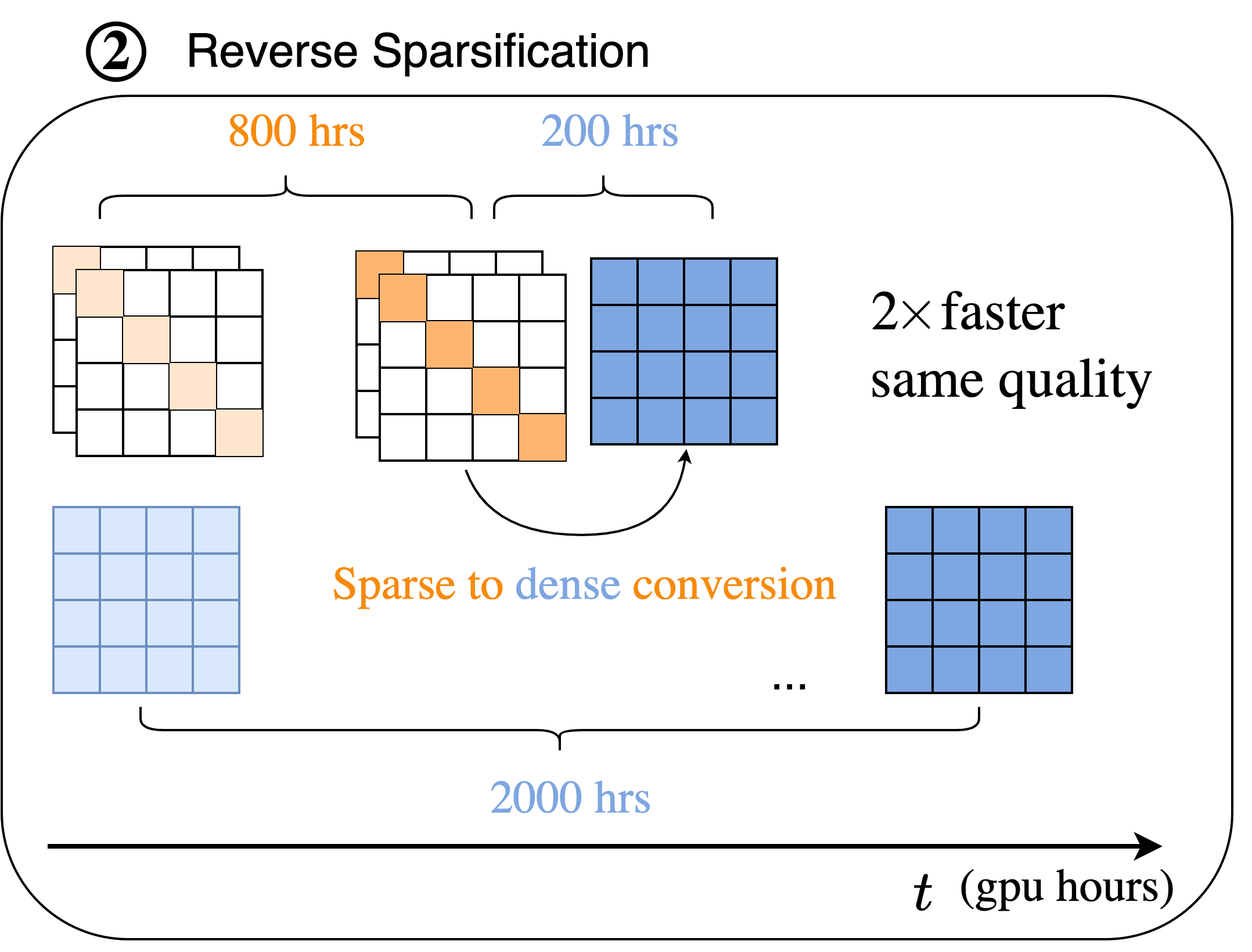}
  \vspace{-0.5em}
  \caption{\label{fig:reverse_sparsification}With the ``reverse sparsification'' process, Monarch matrices can speed up GPT-2 training by 2x.}
  \vspace{-0.5em}
\end{figure}

Specifically, observe that if $\vM \in \MMS$, we can write $\vM = (\vP\vL \vP^\top \vR) (\vR'^* \vP\vL'^* \vP^\top) = (\vP\vL_1\vP^\top) \vR (\vP\vL_2\vP^\top)$ for block-diagonal $\vL_1,\vL_2,\vR$ and the permutation $\vP$ of \cref{def:Monarch}.
Then, we can compute $\vL_1,\vL_2,\vR$ in such a factorization under Assumption \ref{assump:a1}, as stated in \cref{thm:Monarch_recovery}.  (Note that the factorization is not unique.)
\begin{assumption}
\label{assump:a1}
Assume that (1) $\vM \in \MMS$ is invertible and (2) $\vM$ can be written as $(\vP\vL_1\vP^\top)\vR(\vP\vL_2\vP^\top)$ where the blocks of $\vR$ have no zero entries.%
\end{assumption}%

\begin{theorem}\label{thm:Monarch_recovery}
Given an $n \times n$ matrix $\vM \in \MMS$ satisfying Assumption \ref{assump:a1}, there is an $O(n^{5/2})$-time algorithm to find its Monarch factors $\vL_1, \vR, \vL_2$.
\end{theorem}

\mbox{To understand how to do this, define $\tilde{\vM} = \vP^\top \vM\vP$} \mbox{and observe that $\tilde{\vM}\, = \, \vL_1 (\vP\vR\vP^\top) \vL_2 \, =$} \newline
{
\setstretch{0.67}
\setlength\arraycolsep{1.2pt}
${}\hspace{-0.35em}\small \lt \begin{array}{ccccc} \vA_1 \\ & \vA_2 \\ & & \ddots \\ & & & \vA_{m} \end{array}\rt
\lt \begin{array}{cccc} \vD_{11} & \vD_{12} & \dots & \vD_{1m} \\ \vD_{21} & \vD_{22} & \dots & \vD_{2m} \\ \ddots & \ddots & \ddots & \ddots  \\ \vD_{m1} & \vD_{m2} & \dots & \vD_{mm} \end{array}\rt
\lt \begin{array}{ccccc} \vC_1 \\ & \vC_2 \\ & & \ddots \\ & & & \vC_{m} \end{array}\rt$
}

where $m = \sqrt{n}$, the $\vA_i$'s and $\vC_j$'s denote the $m \times m$ diagonal blocks of $\vL_1,\vL_2$ respectively, and each $\vD_{ij}$ is an $m \times m$ diagonal matrix. If we write $\MM$ as a block matrix with $m \times m$ blocks each of size $m \times m$, then we see that the block $\MM_{ij}$ is equal to $\vA_i \vD_{ij} \vC_j$. Notice that $\vM$ is invertible only if all the $\vA_i$'s and $\vC_j$'s are (since if any one of these is singular, then $\vL_1$ or $\vL_2$ is singular).

Thus, our goal is to find matrices $\hat{\vA}_1,\dots,\hat{\vA}_m, \hat{\vC}_1,\dots,\hat{\vC}_m$ and \emph{diagonal} matrices $\hat{\vD}_{11},\dots,\hat{\vD}_{mm}$ such that $\MM_{ij} = \hat{\vA}_i \hat{\vD}_{ij} \hat{\vC}_j$ for all $i,j$; this represents a valid Monarch factorization of $\vM$.

To provide intuition for how to do this, let's analyze a simple case in which all the $\vD_{ij}$'s are the identity matrix. Then we have the set of equations $\vA_i \vC_j = \MM_{ij}$. Again assume the $\vA_i$'s and $\vC_j$'s are invertible, so each $\MM_{ij}$ is as well. Suppose we set $\hat{\vC}_1 = \vI$ (identity matrix). Then we can immediately read off $\hat{\vA}_i = \MM_{i1}$ for all $i$. We can then set $\hat{\vC}_j = \hat{\vA}_1^{-1}\MM_{1j}$ for all $j$.
Let's now check that this strategy gives a valid factorization, i.e., that $\MM_{ij} = \hat{\vA}_i \hat{\vC}_j$ for all $i,j$. We have $\hat{\vA}_i \hat{\vC}_j = \MM_{i1} \MM_{11}^{-1} \MM_{1j}$. Recalling that in the ``true'' factorization we have $\MM_{ij} = \vA_i \vC_j$, this equals $(\vA_i \vC_1) (\vA_1 \vC_1)^{-1} (\vA_1 \vC_j) = \vA_i \vC_j$, as desired.

In the general case, we must deal with the diagonal $\vD_{ij}$ matrices as well. We will no longer be able to freely set $\hat{\vC}_1 = \vI$. However, once we find a proper choice of $\hat{\vC}_1$, we can use it to find all the $\hat{\vA}_i$'s and $\hat{\vC}_j$'s. We can find such a $\hat{\vC}_1$ via the idea of \emph{simultaneous diagonalization}; for space reasons,
we defer a full description of our algorithm (\cref{alg:mm_recovery}), and its analysis, to Appendix \ref{sec:proofs}.

\begin{figure}[t]
  \centering
  \includegraphics[width=.45\textwidth]{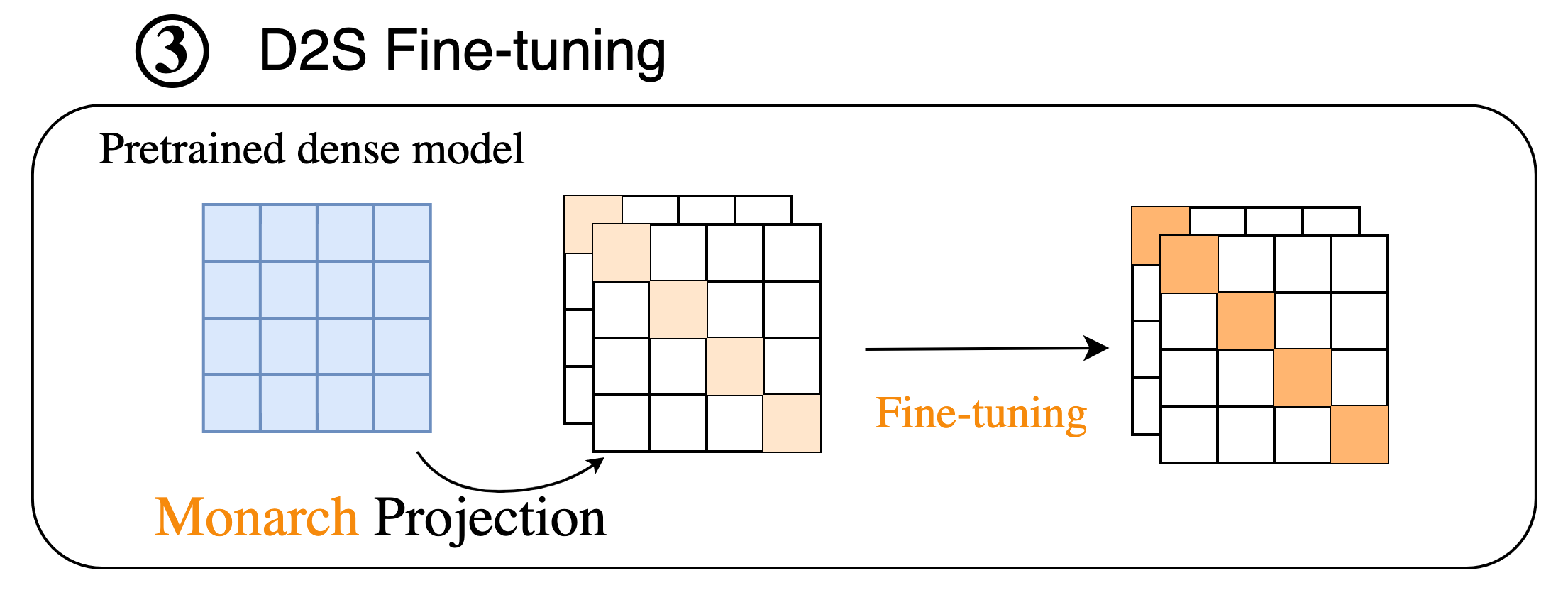}
  \vspace{-1em}
  \caption{\label{fig:monarch_projection}With~\cref{alg:project} for our Monarch parameterization, we can convert a pretrained model into a model with Monarch weight matrices and speed up downstream fine-tuning.}
  \vspace{-1.0em}
\end{figure}

\section{Using Monarch Matrices in Model Training}
\label{sec:method}

We can use our class of Monarch matrices to parameterize weight
matrices of deep learning models in several settings.

\begin{itemize}[leftmargin=*,nosep,nolistsep]
  \item In the \textbf{E2E sparse training} setting, we replace the dense weight
  matrices of a baseline model with Monarch matrices with the same dimension,
  initialize them randomly, and train as usual.
  Most of our baseline models are Transformers, and we replace the projection
  matrices in the attention blocks, along with the weights of the feed-forward
  network (FFN) blocks, with Monarch matrices.
  The Monarch parameterization is differentiable, and we rely on
  autodifferentiation to train with first-order methods such as Adam~\citep{kingma2014adam}.

  \item In the \textbf{S2D training} setting, we first replace the dense weight
  matrices of a baseline model with Monarch matrices, then train the sparse
  model for about 90\% of the usual number of iterations.
  We then convert the Monarch matrices to dense matrices (by simply multiplying
  the factors $L$ and $R$ along with permutations), and continue training for
  the remaining 10\% of the iterations.
  Compared to dense end-to-end training, we train for the same number of
  iterations, but the first 90\% of the iterations are faster due to the
  hardware efficiency of Monarch matrices.

  \item In the \textbf{D2S fine-tuning} setting, we start with a dense
  pretrained model (e.g., BERT), and project the dense weight matrices (e.g., in
  the attention blocks and FFN blocks) on the set of Monarch matrices using the
  algorithm in \cref{subsec:projection}.
  We then fine-tune the resulting model on downstream tasks (e.g., GLUE), using
  first-order methods.
\end{itemize}
We typically set the number of blocks in the block-diagonal matrices to be between 2 and 4 based on the parameter budgets (25\% -- 50\% of the dense model).

\section{Experiments}
\label{sec:experiments}

We validate our approach empirically, showing that our Monarch matrix parametrization achieves a favorable efficiency--accuracy tradeoff compared to baselines on a wide range of domains (text, images, PDEs, MRI), in three settings (E2E training, S2D training, and D2S fine-tuning):
\begin{itemize}[leftmargin=*,nosep,nolistsep,noitemsep]
\item
In \cref{subsec:benchmark_tasks}, on image classification and language modeling benchmarks, such as ViT / MLP Mixer on ImageNet and GPT-2 on Wikitext-103, Monarch is 2$\times$ faster to train than dense models, while achieving the same accuracy / perplexity. In \cref{subsec:pde_mri}, in scientific and medical domains where special transforms (Fourier) are common, Monarch outperforms Fourier transform based methods on PDE solving, with up to 40\% lower error, and on MRI reconstruction attains up to 15\% higher pSNR and 3.8\% higher SSIM.
\item In \cref{subsec:pde_mri}, we show that on the large OpenWebText dataset, reverse sparsification (training with Monarch weight matrices for most of the time, then transitioning to dense weight matrices) speeds up the pretraining of GPT-2 models by 2$\times$ compared to the dense model, with no loss in upstream or downstream quality.
Moreover, reverse sparsification speeds up BERT pretraining by 23\% even compared to the implementation from Nvidia that set the MLPerf~\citep{mattson2020mlperf} 1.1 record.
\item In \cref{subsec:finetuning}, as a proof of concept, we demonstrate that our Monarch approximation algorithm can improve fine-tuning efficiency for pretrained models. We show that compressing BERT to a Monarch matrix model performs comparably to a finetuned dense model on GLUE, with 2$\times$ fewer parameters and 1.7$\times$ faster finetuning speed.
\end{itemize}

\subsection{End-to-End Training}
\label{subsec:e2e_training}
\subsubsection{Benchmark Tasks: Image Classification, Language Modeling}
\label{subsec:benchmark_tasks}

We show that replacing dense matrices with Monarch matrices in ViT, MLP-Mixer, and
GPT-2 can speed up training by up to 2$\times$ without sacrificing model quality in~\cref{table:pretrain,table:gpt_pretrain}.

\textbf{Setup.} We use the popular vision benchmark, ImageNet~\citep{deng2009imagenet}. We choose recent popular Vision Transformer~\citep{dosovitskiy2020image}, and MLP-Mixer~\citep{tolstikhin2021mlp} as representative base dense models.
For language modeling, we evaluate GPT-2~\citep{radford2019language} on WikiText-103~\citep{merity2016pointer}.

\begin{table}[h]
  \small
  \centering
  \vspace{-2mm}
  \caption{\label{table:pretrain}The performance of Monarch matrices and ViT / MLP-Mixer on ImageNet, including the number of parameters and FLOPs. We measure the Top-1 accuracy and the training time speedup compared to the corresponding dense model. %
  \vspace{2mm}
  }
  \iftoggle{arxiv}{}{
  \resizebox{\linewidth}{!}
  }
  {
  \setlength{\tabcolsep}{3pt}
  \vspace{3em}
  \begin{tabular}{@{}c||ccccccc@{}}
  \specialrule{.15em}{.05em}{.05em}
    Model&\multicolumn{1}{c}{ImageNet acc.}&\multicolumn{1}{c}{Speedup} &\multicolumn{1}{c}{Params} & \multicolumn{1}{c}{FLOPs} \\
    \specialrule{.15em}{.05em}{.05em}
    Mixer-S/16& 74.0& - & 18.5M & 3.8G \\
    Monarch-Mixer-S/16& 73.7& 1.7$\times$ & 7.0M & 1.5G \\
    Mixer-B/16& 77.7& - & 59.9M & 12.6G \\
    Monarch-Mixer-B/16& 77.8& 1.9$\times$ & 20.9M & 5.0G \\
    \specialrule{.15em}{.05em}{.05em}
    ViT-S/16& 79.4 & - & 48.8M & 9.9G \\
    Monarch-ViT-S/16& 79.1 & 1.9$\times$ & 19.6M & 3.9G \\
    ViT-B/16& 78.5 & - & 86.6M  & 17.6G \\
    Monarch-ViT-B/16& 78.9 & 2.0$\times$ & 33.0M & 5.9G \\
    \specialrule{.15em}{.05em}{.05em}
  \end{tabular}
  }
\end{table}

\begin{table}[h]
  \small
  \centering
  \vspace{-3mm}
  \caption{\label{table:gpt_pretrain} Performance of Monarch matrices and GPT-2-Small/Medium on WikiText-103, including the \# of parameters and FLOPs. Monarch achieves similar perplexity (ppl) but 2.0$\times$ faster.}
  \vspace{1mm}
  \iftoggle{arxiv}{}{
    \resizebox{0.95\linewidth}{!}
  }
  {
\setlength{\tabcolsep}{5pt}
\begin{tabular}{c||cccc}
\specialrule{.15em}{.05em}{.05em}
\multirow{1}{*}{{ Model} } & \multicolumn{1}{c}{\multirow{1}{*}{PPL}}
                              & \multicolumn{1}{c}{\multirow{1}{*}{Speedup}}
                              & \multicolumn{1}{c}{\multirow{1}{*}{Params}}
                              & \multicolumn{1}{c}{\multirow{1}{*}{FLOPs}}\\
\specialrule{.15em}{.05em}{.05em}
GPT-2-Small &  20.6 & - & 124M& 106G\\
Monarch-GPT-2-Small& 20.7  & 1.8$\times$ &72M & 51G\\
\specialrule{.15em}{.05em}{.05em}
GPT-2-Medium &  20.9 & - & 355M& 361G\\
Monarch-GPT-2-Medium& 20.3  & 2.0$\times$ &165M & 166G\\
\specialrule{.15em}{.05em}{.05em}
\end{tabular}
}
\vspace{-2mm}
\end{table}

\subsubsection{PDE solving and multi-coil MRI reconstruction}
\label{subsec:pde_mri}

Many scientific or medical imaging tasks rely on specialized transforms such as the
Fourier transform.
We show that replacing the fixed Fourier transform with the more expressive
Monarch matrices yields higher model quality (lower reconstruction error) with
comparable model speed.

\textbf{Solving PDEs with Monarch Neural Operators.}
We follow the experimental setting in FNO~\citep{li2020fourier} and apply a Monarch--based neural operator to the task of solving the Navier--Stokes PDE. Compared to baseline U-Nets~\citep{ronneberger2015u}, TF-Nets~\citep{wang2020towards}, ResNets~\citep{he2016deep} and FNOs~\cite{li2020fourier}, neural operators based on Monarch improve solution accuracy across spatial resolutions by up to $40\%$ (Table \ref{table:pde}).

\paragraph{Non-periodic boundary conditions.} Traditional spectral methods based on Fourier transform work best with periodic boundary conditions and forcing terms. However, PDEs of practical interest often exhibit non--periodic or even unknown boundary conditions. Monarch operators are not constrained to the Fourier transform and can thus still learn the solution operator with excellent accuracy.

\begin{table}[h!] 
\scriptsize
\vspace{-4mm}
\caption{\label{table:pde}Benchmarks on Navier-Stokes (fixing resolution 64 × 64 for both training and testing).
Decreasing the viscosity coefficient $\nu$ makes the dynamics more chaotic.
}
\vspace{1mm}
\centering
\iftoggle{arxiv}{}{
  \resizebox{0.9\linewidth}{!}
}
{
\renewcommand{\arraystretch}{1}
\begin{tabular}{ c||ccc }
\specialrule{.15em}{.05em}{.05em}
Model & $v = 10^{-3}$  &  $v = 10^{-4}$ & $v = 10^{-5}$\\
\specialrule{.15em}{.05em}{.05em}
U-Net & 0.025  & 0.205  &   0.198\\
TF-Net  & 0.023  & 0.225 &  0.227 \\
ResNet & 0.070 &  0.287 &  0.275 \\
FNO & 0.017  & 0.178 & 0.155\\
Monarch-NO & \textbf{0.010} & \textbf{0.145} & \textbf{0.136} \\
\specialrule{.15em}{.05em}{.05em}
\end{tabular}
}
\textbf{\vspace{-3mm}}
\end{table}

\textbf{Accelerated MRI Reconstruction.} We characterize the utility of Monarch-based FFT operations for accelerated MRI reconstruction, a task which requires methods with both structured Fourier operators and dealiasing properties to recover high quality images. On the clinically-acquired 3D MRI SKM-TEA dataset \citep{desai2021skm}, Monarch-SENSE (mSENSE) enhances image quality by over 1.5dB pSNR and 2.5\% SSIM compared to zero-filled SENSE and up to 4.4dB and 3.8\% SSIM compared to U-Net baselines in data-limited settings. Setup details are available in~\cref{sec:experiment_details_mri}.

\paragraph{Expressive FFT.} By definition, standard IFFT in zero-filled SENSE cannot dealias the signal, resulting in artifacts in the reconstructed image. mSENSE replaces the inverse FFT (IFFT) operation in standard SENSE with learnable Monarch matrices. Thus, mSENSE preserves the structure of the Fourier transform while learning to reweight frequencies to suppress aliasing artifacts. Across multiple accelerations, mSENSE achieved up to +1.5dB and 2.5\% improvement in peak signal-to-noise ratio (pSNR) and structural similarity (SSIM), respectively (Table~\ref{table:mri}).

\paragraph{Data Efficiency.} While CNNs have shown promise for MRI reconstruction tasks, training these networks requires extensive amounts of labeled data to avoid overfitting. However, large data corpora are difficult to acquire in practice. mSENSE can be trained efficiently with limited supervised examples. In few shot settings, mSENSE can outperform U-Net by +4.4dB ($\approx$15\%) and 3.8\% SSIM (Table~\ref{table:mri-data-limited}).

\begin{table}[h!] 
\scriptsize
\vspace{-3mm}
\caption{\label{table:mri}Mean $\pm$ standard error of the mean of conventional and Monarch-SENSE (mSENSE) on dual-echo (E1,E2) MRI reconstruction at multiple acceleration factors (Acc.).
}
\vspace{1mm}
\centering
\iftoggle{arxiv}{}{
  \resizebox{\linewidth}{!}
}
{
\renewcommand{\arraystretch}{1.2}
\begin{tabular}{c||ccccc}
\specialrule{.15em}{.05em}{.05em}
  & & \multicolumn{2}{c}{pSNR (dB) ($\uparrow$)} & \multicolumn{2}{c}{SSIM ($\uparrow$)} \\
  Acc. & Model &             E1 &             E2 &                E1 &                E2 \\
\specialrule{.15em}{.05em}{.05em}
\multirow{2}{*}{2} & SENSE &  32.8$\pm$0.2 &  35.4$\pm$0.2 &  0.871$\pm$0.003 &  0.865$\pm$0.003 \\
  & mSENSE &  \textbf{34.3$\pm$0.2} &  \textbf{36.6$\pm$0.2} &  \textbf{0.886$\pm$0.002} &  \textbf{0.882$\pm$0.003} \\
\specialrule{.15em}{.05em}{.05em}
\multirow{2}{*}{3} & SENSE &  30.9$\pm$0.2 &  33.5$\pm$0.2 &  0.819$\pm$0.004 &  0.795$\pm$0.004 \\
  & mSENSE &  \textbf{32.3$\pm$0.2} &  \textbf{34.6$\pm$0.2} &  \textbf{0.843$\pm$0.003} &  \textbf{0.820$\pm$0.004} \\
\specialrule{.15em}{.05em}{.05em}
\multirow{2}{*}{4} & SENSE &  30.1$\pm$0.2 &  32.8$\pm$0.2 &  0.789$\pm$0.004 &  0.753$\pm$0.005 \\
  & mSENSE &  \textbf{31.2$\pm$0.2} &  \textbf{33.5$\pm$0.2} &  \textbf{0.812$\pm$0.003} &  \textbf{0.767$\pm$0.005} \\
\specialrule{.15em}{.05em}{.05em}
\end{tabular}
}
\end{table}

\begin{table}[h!] 
\scriptsize
\vspace{-5mm}
\caption{\label{table:mri-data-limited}Impact of number of training examples ($N$) on dual-echo MRI reconstruction at 2x acceleration.
}
\vspace{1mm}
\centering
\iftoggle{arxiv}{}{
  \resizebox{\linewidth}{!}
}
{
\renewcommand{\arraystretch}{1.2}
\begin{tabular}{c||ccccc}
\specialrule{.15em}{.05em}{.05em}
  &  & \multicolumn{2}{c}{pSNR (dB) ($\uparrow$)} & \multicolumn{2}{c}{SSIM ($\uparrow$)} \\
  $N$ & Model &            E1 &            E2 &               E1 &               E2 \\
\specialrule{.15em}{.05em}{.05em}
N/A & SENSE &  32.8$\pm$0.2 &  35.4$\pm$0.2 &  0.871$\pm$0.003 &  0.865$\pm$0.003 \\
\specialrule{.15em}{.05em}{.05em}
\multirow{2}{*}{1} & U-Net &  29.4$\pm$0.2 &  34.4$\pm$0.3 &  0.848$\pm$0.004 &  0.857$\pm$0.004 \\
  & mSENSE &  \textbf{33.8$\pm$0.2} &  \textbf{36.0$\pm$0.2} &  \textbf{0.886$\pm$0.003} &  \textbf{0.867$\pm$0.003} \\
\specialrule{.15em}{.05em}{.05em}
\multirow{2}{*}{2} & U-Net &  29.9$\pm$0.3 &  35.1$\pm$0.3 &  0.858$\pm$0.003 &  0.871$\pm$0.003 \\
  & mSENSE &  \textbf{34.0$\pm$0.2} &  \textbf{36.4$\pm$0.2} &  \textbf{0.883$\pm$0.002} &  \textbf{0.877$\pm$0.003} \\
\specialrule{.15em}{.05em}{.05em}
\multirow{2}{*}{3} & U-Net &  31.0$\pm$0.3 &  35.2$\pm$0.3 &  0.866$\pm$0.003 &  0.867$\pm$0.004 \\
  & mSENSE &  \textbf{33.9$\pm$0.2} & \textbf{ 36.5$\pm$0.2} &  \textbf{0.882$\pm$0.002} & \textbf{0.878$\pm$0.003} \\
\specialrule{.15em}{.05em}{.05em}
\multirow{2}{*}{5} & U-Net &  31.4$\pm$0.3 &  35.6$\pm$0.2 &  0.877$\pm$0.002 &  0.870$\pm$0.003 \\
  & mSENSE &  \textbf{33.9$\pm$0.2} &  \textbf{36.5$\pm$0.2} &  \textbf{0.881$\pm$0.002} &  \textbf{0.877$\pm$0.003} \\
\specialrule{.15em}{.05em}{.05em}
\end{tabular}
}
\end{table}

\subsection{Sparse-to-Dense Training (reverse sparsification)}
\label{subsec:s2d_training}
\paragraph{GPT-2 pretraining.}
On the large OpenWebtext dataset~\citep{Gokaslan2019OpenWeb}, we train a GPT-2 model with Monarch weight
matrices for 90\% of the training iterations, then relax the constraint on the
weight matrices and train them as dense matrices for the remaining 10\% of the
iterations.
We call this technique ``reverse sparsification.''
Previous sparse training techniques often don't speed up training, whereas our
hardware-efficient Monarch matrices do.
Therefore we can use them as an intermediate step to pretrain a large language
model (GPT-2) in 2$\times$ less time. We also evaluate its downstream quality on zero-shot generation from~\citep{eval-harness} and classification tasks from~\citep{zhao2021calibrate}, achieving comparable performance to the dense counterparts (\cref{table:gpt_finetune}). 

\begin{table}[h]
  \small
  \centering
  \vspace{-3mm}
  \caption{\label{table:gpt_finetune}The performance (accuracy) of GPT-2-medium trained with Monarch reverse sparsification and with conventional dense training on text classification benchmarks.}
  \setlength{\tabcolsep}{5pt}
  \vspace{1em}
  \iftoggle{arxiv}{}{
    \resizebox{\linewidth}{!}
  }
  {
  \begin{tabular}{@{}c||ccc@{}}
    \specialrule{.15em}{.05em}{.05em}
    Model&\multicolumn{1}{c}{OpenWebText (ppl)}&\multicolumn{1}{c}{Speedup}& \multicolumn{1}{c}{Classification (avg acc)} \\
    \specialrule{.15em}{.05em}{.05em}
    GPT-2m& 18.0 & - & 38.9 \\
    Monarch-GPT-2m& 18.0 & 2$\times$ & 38.8 \\
    \specialrule{.15em}{.05em}{.05em}
  \end{tabular}
  }
  \vspace{-3mm}
\end{table}

In \cref{fig:reverse_sparsification_bar}, we show the training time of the dense GPT-2 model, along with
the Monarch GPT-2 model.
After training the Monarch model for 90\% of the time, in the
last 10\% of the training steps, by transitioning to dense weight matrices, the model is able to reach the same 
performance of another model that was trained with dense weight matrices from
scratch.
By training with Monarch matrices for 90\% of the time, we reduce the total training time by 2$\times$.

\paragraph{BERT pretraining.}
On the Wikipedia + BookCorpus datasets~\citep{zhu2015aligning}, we train a BERT-large model with Monarch weight matrices for 70\% of the time and transition to dense weight matrices for the remaining 30\% of the time, which yields the same pretraining loss as conventional dense training.
In \cref{table:bert_speed}, we compare the total training time to several baseline implementations: the widely-used implementation from HuggingFace~\citep{wolf-etal-2020-transformers}, the more optimized implementation from Megatron~\citep{shoeybi2019megatron}, and the most optimized implementation we know of from Nvidia that was used to set MLPerf 1.1 training speed record. Our method is 3.5x faster than HuggingFace and 23\% faster than Nvidia's MLPerf 1.1 implementation\footnote{Our result is not an official MLPerf submission. We train BERT for both phase 1 (sequence length 128) and phase 2 (sequence length 512) according to the standard BERT training recipe\cite{devlin2018bert}, while MLPerf only measures training time for phase 2.}.
Experiment details are in~\cref{subsec:bert_details}.

\begin{table}[h]
  \small
  \centering
  \caption{\label{table:bert_speed}The total training time of BERT-large trained with Monarch reverse sparsification and with conventional dense training on 8 A100-40GB GPUs (DGX A100). Training consists of two phases, phase 1 with sequence length 128 and phase 2 with sequence length 512. Monarch training is 3.5x faster than HuggingFace and 23\% faster than Nvidia's MLPerf 1.1 implementation.}
  \vspace{1em}
  \iftoggle{arxiv}{}{
    \resizebox{\linewidth}{!}
  }
  {
    \begin{tabular}{@{}c||c@{}}
      Implementation & Training time (h)  \\ \hline
      HuggingFace &  84.5 \\
      MegaTron & 52.5 \\
      Nvidia MLPerf 1.1 & 30.2 \\
      Nvidia MLPerf 1.1 + DeepSpeed & 29.3 \\
      Monarch (ours) & \textbf{23.8} \\
    \end{tabular}
  }
  \vspace{-3mm}
\end{table}

\subsection{Dense-to-Sparse Fine-tuning}
\label{subsec:finetuning}

\begin{figure}[t]
  \centering
  \vspace{-3mm}
  \includegraphics[width=.3\textwidth]{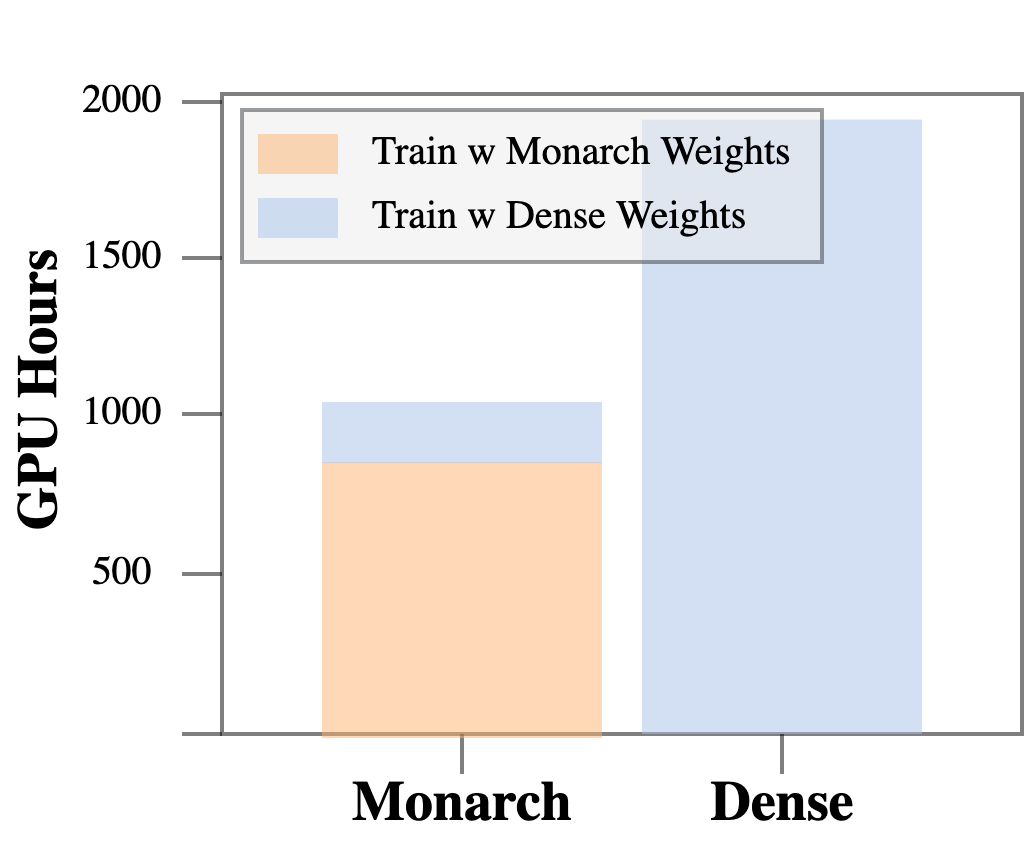}
  \vspace{-3mm}
  \caption{\label{fig:reverse_sparsification_bar}Time required (in A100 GPU hours) to reach the same perplexity (18.0)
    for GPT-2-small on OpenWebText.
    With ``reverse sparsification'', Monarch can speed up
    GPT-2 training by 2$\times$.\vspace{-1em}}
\end{figure}

We show that our Monarch approximation algorithm allows us to efficiently use
pretrained models, such as speeding up BERT finetuning on GLUE.

\paragraph{BERT finetuning.}
We take the BERT pretrained weights, approximate them with Monarch matrices,
and finetune the resulting model on the 9 GLUE tasks.
The results in \cref{table:bert_glue} shows that we obtain a Monarch finetuned
model with similar quality to the dense BERT model, but with 1.7$\times$ faster
finetuning speed.
This serves as a proof of concept, and we expect further speedup if additional model compression techniques are applied (e.g., quantization, kernel fusion).

\begin{table}[h]
  \small
  \centering
  \vspace{-5mm}
  \caption{\label{table:bert_glue}The performance of Monarch matrices in
    finetuning BERT on GLUE.}
  \setlength{\tabcolsep}{5pt}
  \vspace{1em}
  \iftoggle{arxiv}{}{
    \resizebox{\linewidth}{!}
  }
  {
  \begin{tabular}{@{}c||ccccccc@{}}
  \specialrule{.15em}{.05em}{.05em}
    Model&\multicolumn{1}{c}{GLUE (avg)}&\multicolumn{1}{c}{Speedup} &\multicolumn{1}{c}{Params} & \multicolumn{1}{c}{FLOPs} \\
    \specialrule{.15em}{.05em}{.05em}
    BERT-base & 78.6& - & 109M & 11.2G \\
    Monarch-BERT-base& 78.3& 1.5$\times$ & 55M & 6.2G  \\
    BERT-large & 80.4 & - & 335M & 39.5G \\
    Monarch-BERT-large & 79.6 & 1.7$\times$ & 144M & 14.6G  \\
    \specialrule{.15em}{.05em}{.05em}
  \end{tabular}
  }
  \vspace{-3mm}
\end{table}

\section{Conclusion}
\label{sec:conclusion}
We propose Monarch, a novel matrix parameterization that inherits the expressiveness of butterfly matrices and thus can represent many fast transforms.
Our parameterization leverages optimized batch matrix multiply routines on GPUs, yielding up to 2$\times$ speedup compared to dense matrix multiply.
We derive an efficient algorithm for projecting an arbitrary dense matrix on the set of Monarch factors.
Our algorithm allows us to easily fine-tune a pretrained model into a model with Monarch weight matrices.
As a result, Monarch matrices unlock new ways for faster end-to-end training, sparse-to-dense training, and dense-to-sparse fine-tuning of large neural networks.
By making structured matrices practical, our work is a first step towards unlocking tremendous performance improvements in applying sparse models to wide-ranging ML applications (including science and medicine).
We anticipate this work can inspire more future work on advancing machine learning models for interdisciplinary research with limited computational resources.

\section*{Acknowledgments}

We thank Laurel Orr, Xun Huang, Trevor Gale, Jian Zhang, Victor Bittorf, Sarah Hooper, Neel Guha, and Michael Zhang for their helpful discussions and feedback on early drafts of the paper.

We gratefully acknowledge the support of NIH under No.\ U54EB020405 (Mobilize), NSF under Nos.\ CCF1763315 (Beyond Sparsity), CCF1563078 (Volume to Velocity), and 1937301 (RTML); ARL under No.\ W911NF-21-2-0251 (Interactive Human-AI Teaming); ONR under No.\ N000141712266 (Unifying Weak Supervision); ONR N00014-20-1-2480: Understanding and Applying Non-Euclidean Geometry in Machine Learning; N000142012275 (NEPTUNE); NXP, Xilinx, LETI-CEA, Intel, IBM, Microsoft, NEC, Toshiba, TSMC, ARM, Hitachi, BASF, Accenture, Ericsson, Qualcomm, Analog Devices, Google Cloud, Salesforce, Total, the HAI-GCP Cloud Credits for Research program,  the Stanford Data Science Initiative (SDSI), and members of the Stanford DAWN project: Facebook, Google, and VMWare. The U.S.\ Government is authorized to reproduce and distribute reprints for Governmental purposes notwithstanding any copyright notation thereon. Any opinions, findings, and conclusions or recommendations expressed in this material are those of the authors and do not necessarily reflect the views, policies, or endorsements, either expressed or implied, of NIH, ONR, or the U.S.\ Government.

\bibliography{ref}
\bibliographystyle{icml2022}

\newpage
\appendix
\onecolumn

\section{Extended Related Work}
\label{app:related}
In this section, we extend the related works referenced in the main paper and discuss them in detail.
\paragraph{Sparse Training.} Our work is loosely related to neural network pruning. By iteratively eliminating neurons and connections, pruning has seen great success in compressing complex models. \citet{han2015deep,han2015learning} put forth two naive but effective algorithms to compress models up to 49x and maintain comparable accuracy. \citet{li2016pruning} employ filter pruning to reduce the cost of running convolution models up to 38 $\%$, \citet{NIPS2017_a51fb975} prunes the network at runtime, hence retaining the flexibility of the full model. \citet{dong2017learning} prunes the network locally in a layer by layer manner.  \citet{sanh2020movement} prunes with deterministic first-order information, which is more adaptive to pretrained model weights. \citet{lagunas2021block} prunes transformers models with block sparsity pattern during fine-tuning, which leads to real hardware speed up while maintaining the accuracy. \citet{zhu2017prune} finds large pruned sparse network consistently outperform the small dense networks with the same compute and memory footprints. Although both our and all the pruning methods are aiming to produce sparse models, we differ in our emphasis on the overall efficiency, whereas pruning mostly focuses on inference efficiency and disregards the cost in finding the smaller model.

There has been more recent work on sparse methods that focuses on speeding up
training and not just inference, such as SNFS~\citep{dettmers2019sparse},
RigL~\citep{dettmers2019sparse}, Top-KAST~\citep{jayakumar2021top}.
These methods often focus on FLOP counts, which may not correlate well with
wall-clock time on modern hardware (e.g., GPUs).
Block-sparsity is another approach that exploits the block-oriented nature of
GPUs~\citep{gray2017gpu, child2019generating, guo2020accelerating}.
Sparse models have also been found useful to improve the training process of
dense models.
For example, sparsity can be used to regularize dense models to improve
accuracy~\citep{han2016dsd}, or to alternate between sparse and dense training
to ease deployment~\citep{peste2021ac}.
Our sparse-to-dense reverse sparsification instead focuses on speeding up dense
training, where the sparse model is used for efficiency and not regularization.

In addition, models proposed in our work can be roughly seen as a class of manually constructed lottery tickets. Lottery tickets \citet{frankle2018lottery} are a set of small sub-networks derived from a larger dense network, which outperforms their parent networks in convergence speed and potentially in generalization. A huge number of studies are carried out to analyze these tickets both empirically and theoretically: \citet{morcos2019one} proposed to use one generalized lottery tickets for all vision benchmarks and got comparable results with the specialized lottery tickets; \citet{frankle2019stabilizing} improves the stability of the lottery tickets by iterative pruning; \citet{frankle2020linear} found that subnetworks reach full accuracy only if they are stable against SGD noise during training; \citet{orseau2020logarithmic} provides a logarithmic upper bound for the number of parameters it takes for the optimal sub-networks to exist; \citet{pensia2020optimal} suggests a way to construct the lottery ticket by solving the subset sum problem and it's a proof by construction for the strong lottery ticket hypothesis. Furthermore, follow-up works \citep{liu2020finding, wang2020picking, tanaka2020pruning} show that we can find tickets without any training labels.

\paragraph{Structured matrices and butterfly matrices.}
Structured matrices are those with asymptotically fast matrix-vector
multiplication algorithm ($o(n^2)$ time complexity) and few parameters ($o(n^2)$
space complexity).
Common examples include sparse \& low-rank matrices, and fast transforms such as
Fourier transform, Chebyshev transform, Legendre transform, and more generally
orthogonal polynomial transforms.
These transforms have been widely used in data preprocessing (e.g., DFT in
speech processing~\citep{jurafsky2014speech}) and kernel
approximation~\citep{le2013fastfood,yu2016orthogonal}.
Many generalizations of these transforms have been used in machine learning to
replace dense weight
matrices~\citep{sindhwani2015structured,thomas2018learning,gu2020hippo}.
\citet{desa2018two} shows that any structured matrix (in the form of arithmetic
circuits) can be written as product of sparse matrices,
and~\citet{dao2020kaleidoscope} shows that products of butterfly matrices can
represent these structured matrices almost optimally in terms of runtime and
memory.
The class of butterfly matrices~\citep{parker1995random} have also been used in
kernel models~\citep{munkhoeva2018quadrature, choromanski2019unifying} and deep
learning models~\citep{vahid2020butterfly,lin2021deformable,
  ailon2021sparse}.

\paragraph{Neural Operators for PDEs.}

Deep learning has found application in the domain of differential equations and scientific computing \cite{rackauckas2020universal}, with methods developed for prediction and control problems \cite{kidger2020neural,massaroli2021differentiable}, as well as acceleration of numerical schemes \cite{poli2020hypersolvers,jolicoeur2021gotta}. Specific to the \textit{partial differential equations} (PDEs) are approaches designed to learn solution operators \cite{raissi2019physics,fan2020solving,li2020fourier}, and hybridized solvers \cite{kochkov2021machine}, evaluated primarily on classical fluid dynamics.

The promise of these approaches is to offer, at the cost of an initial training procedure, accurate yet faster solutions than an appropriate numerical method tuned for a specific problem, which can then be leveraged for real-time forecasting or within larger feedback loops. Nonetheless, optimal design of neural operators remains an open problem, with most relying on fast Fourier transforms (FFT) or standard dense neural architectures. Instead, neural operators based on Monarch are capable of approximating all fast transforms, thus allowing automated optimization towards a suitable transform on a given PDE problem.

\paragraph{MRI.} Accelerated multi-coil MRI is an essential mechanism for reducing long scan times and making certain scan types feasible. In multi-coil MRI, data is acquired in the spatial Fourier domain (a.k.a \textit{k-space}) across multiple coils (sensors). To reduce scan time, this data is sampled below the required rate for recovering the underlying signal (i.e. Nyquist rate), which results in signal aliasing (see Appendix \ref{sec:experiment_details_mri}). In these settings, direct application of the inverse fast Fourier transform (FFT) cannot suppress aliasing artifacts.

Classical MRI reconstruction approaches supplement the FFT by leveraging shared information across multiple coils and strong analytical priors to regularize image recovery objectives. SENSE-based methods jointly dealias images across multiple coils and reweight the final image based on the spatial sensitivity profile of each coil \citep{pruessmann1999sense}. Compressed sensing promotes image sparsity in transformation domains (e.g. Fourier, wavelet) while enforcing data consistency between the Fourier transform of the reconstructed image and the observed measurements \citep{lustig2007sparse}. Low-rank methods enforce low rank structure across slowly-varying dimensions or local patches in the data \citep{ong2016beyond,ravishankar2017low,haldar2013low}. Additionally, GRAPPA-based techniques optimize kernels to directly interpolate missing k-space samples to promote smoothness in the Fourier domain \cite{griswold2002generalized}. Despite their efficacy, these methods have long reconstruction times, require explicit analytical priors, and require careful hyperparameter fine-tuning.

CNNs have shown promise as a fast-at-inference, learnable alternative to classical MRI reconstruction methods \cite{knoll2020deep}. In supervised learning, fully convolutional networks (e.g. U-Net \citep{ronneberger2015u} or unrolled networks \citep{sandino2020compressed,hammernik2018learning}) learn a mapping between paired zero-filled and fully-sampled, ground truth images. However, supervised methods require a large fully-sampled (labeled) data corpus and are sensitive to distribution drifts due to patient, hardware, and sequence heterogeneity \cite{darestani2021measuring}. To reduce dependence on labeled data, unsupervised methods have used generative adversarial networks \citep{cole2020unsupervised, mardani2018deep}, self-supervised learning \cite{yaman2020self}, dictionary learning \cite{lahiri2021blind}, and untrained networks \cite{darestani2021accelerated}. Despite their 
label efficiency, these techniques still underperform supervised methods and are also sensitive to distribution shift. Recently, a family of semi-supervised reconstruction methods demonstrated label efficiency and robustness to physics-driven perturbations, such as changes in signal-to-noise ratio or patient motion \citep{desai2021noise2recon, desai2021vortex}. However, these methods require large amounts of unlabeled data, which can be difficult to curate in few-shot settings. Thus, despite their success in controlled environments, prospective clinical deployment of these models has been stifled \citep{chaudhari2020prospective}.

In our work, we propose a model with a single FFT-initialized factorized Monarch matrix. Such a matrix can provide the benefits of both a simple linearized transformation like FFT and a learnable mechanism to remove aliasing artifacts resulting from the undersampled k-space. The smaller learnable parameter set may reduce overfitting in data-limited settings while preserving the transformation structure of Fourier matrices. Thus, our approach can be interpreted as a hybrid between analytically-constrained classical methods and data-dependent CNNs.

\section{Notation Review}
Throughout this paper, we use lowercase to denote scalars (e.g., $k$), lowercase boldface to denote vectors (e.g., $\vv$), and uppercase boldface to denote matrices (e.g., $\vA$).

$\vI$ denotes the identity matrix. We use $\vA^\top$ to denote the transpose of a matrix and $\vA^*$ to denote the conjugate transpose of a matrix. All results in this paper apply to matrices over the either the reals $\mathbb{R}$ or the complex numbers $\mathbb{C}$; when the field under consideration can be either one of these, we denote it by $\mathbb{F}$.

We use 1-indexing throughout this paper except where explicitly stated.

\newcommand{\baseb}[3]{\parens{{#1},{#2}}_{{#3}}}
\newcommand{\mx}[1]{\mathbf{#1}}
\newcommand{\floors}[1]{\left \lfloor #1 \right \rfloor}
\newcommand{\parens}[1]{\left( {#1}\right)}

\section{General Monarch Matrix Parametrization}
\label{sec:permutation}

In Section \ref{sec:Monarch_square}, we define a parametrization for square Monarch matrices of different ``block sizes'' (i.e., not necessarily $\sqrt{n}$), and prove some basic properties about them. In Section \ref{sec:Monarch_rect}, we further extend this to define rectangular Monarch matrices, and prove some basic properties about them.

Note: In this section, we use 0-indexing rather than 1-indexing, for notational convenience.

\subsection{General square matrices}
\label{sec:Monarch_square}
\subsubsection{Parametrization}
\label{sec:Monarch_square_param}
In this section, we define a more general Monarch parametrization for square matrices, allowing for different ``block sizes.'' Like \cref{def:Monarch}, the parametrization involves the product of a permuted block-diagonal matrix with another block-diagonal matrix; the difference is that we now allow the matrices $\vL$ and $\vR$ to have diagonal blocks of different sizes. Thus, the permutations applied to $\vL$ (to turn it into a block matrix where each block matrix is diagonal) will correspondingly also be different.

First, in \cref{def:square_r}, we define notation for a class of block-diagonal matrices.

\begin{definition}[Class $\BD\ind{b, n}$]
\label{def:square_r}
Let $b \in (1, n)$ be an integer that divides $n$. For $0\le i< \frac {n}{b}$, let $\mx{R}_{i}\in\F^{b \times b }$ be a $b \times b $ ``block" matrix. Then define the matrix $\vR$ with {\em block size} $b$ as follows:
\begin{equation}
 \label{eq:def-R}
  \vR = \diag\left(\vR_0, \dots, \vR_{\frac {n}{b}-1}\right).
\end{equation}
\end{definition}
(Note that the number of possible nonzero values in $\vR$ is $\frac {n}{b}\cdot b^2=nb$.)
We denote the class of all matrices $\vR$ expressible in this form by $\BD\ind{b, n}$. Note that this class is closed under (conjugate) transposition and contains the identity matrix.

Next, in \cref{def:Matrix L}, we define notation for a class of block matrices whose \emph{blocks} are diagonal.

\begin{definition}[Class $\DB\ind{b,n}$]
\label{def:Matrix L}
Let $b \in (1, n)$ be an integer that divides $n$. For $0 \le i, j < b$, let $\mx{D}_{i,j}\in\F^{b\times b}$ be a $b \times b$ diagonal matrix.
Then let $\vL$ be an $n \times n$ matrix with the following form: 
    \begin{equation}
        	\label{eq:def-L}
    \vL=
    	\begin{bmatrix}
    		\mx{D}_{0,0} & \dots & \mx{D}_{0,\frac{n}{b} -1} \\
    		\vdots & \ddots & \vdots \\
    		\mx{D}_{\frac{n}{b} -1,0} & \dots & \mx{D}_{\frac{n}{b} -1,\frac{n}{b} -1}
    	\end{bmatrix}
    \end{equation}
\end{definition}
(Note that the number of possible nonzero values in $\vL$ is $\parens{\frac nb}^2\cdot b=\frac{n^2}b$.)
We denote the class of all matrices $\vL$ expressible in this form by $\DB\ind{b, n}$. Note that this class is closed under (conjugate) transposition and contains the identity matrix. As we show in \cref{sec:sq-Monarch-properties}, $\vL$ can be written as a block-diagonal matrix with $b$ blocks of size $\ff{n}{b} \times \ff{n}{b}$ (i.e., a matrix in $\BD\ind{\frac{n}{b}, \, n}$), multiplied on the left and right with appropriate permutation matrices.
We denote the class of all matrices $\vL$ expressible in this form by $\DB\ind{b, n}$. Note that this class is closed under (conjugate) transposition. As we show in \cref{sec:sq-Monarch-properties}, $\vL$ can be written as a block-diagonal matrix with $b$ blocks of size $\ff{n}{b} \times \ff{n}{b}$ (i.e., a matrix in $\BD\ind{\frac{n}{b}, \, n}$), multiplied on the left and right with appropriate permutation matrices.

Using these two definitions, we define the class of Monarch matrices with a given block size.
\begin{definition}[Class $\M\ind{b,n}$]
\label{def:block_Monarch}
Let $b \in (1, n)$ be an integer that divides $n$. A \emph{Monarch matrix} of size $n \times n$ and ``block size $b$'' is a matrix of the form: 
    \begin{equation}
        	\label{eq:Monarch-general}
    \vM= \vL \vR
    \end{equation}
    where $\vL \in \DB\ind{b,n}$ and $\vR \in \BD\ind{b,n}$.
\end{definition}
We denote the class of all matrices $\vM$ expressible in this form by $\M\ind{b, n}$. Observe that when $b = \sqrt{n}$, this is exactly the matrix class $\M\ind{n}$ in \cref{def:Monarch}. (In other words, $\M\ind{n}$ is shorthand for $\M\ind{\sqrt{n}, n}$.) Note that a matrix in $\M\ind{b,n}$ is represented by $\frac{n^2}{b} + nb$ parameters.

We remark that $\M\ind{b,n} \supset \B\ind{n}$ for all block sizes $b \in (1, n)$ that divide $n$.

Based on \cref{def:block_Monarch}, we define the classes $\M\M^{*(b,n)}$ and $\M^*\M^{(b,n)}$::
\begin{definition}[Class $\M\M^{*(b,n)}$, $\M^*\M^{(b,n)}$]
\label{def:block_MM}
Let $b \in (1, n)$ be an integer that divides $n$ and suppose $\vM_1, \vM_2 \in \M^{(b,n)}$. We define $\M\M^{*(b,n)}$ to be the the class of all matrices $\vM$ expressible in the form $\vM= \vM_1 \vM_2^*$. \newline
We define $\M^*\M^{(b,n)}$ to be the the class of all matrices $\vM$ expressible in the form $\vM= \vM_1^* \vM_2$.
\end{definition}
Observe that when $b = \sqrt{n}$, $\M\M^{*(b,n)}$ is exactly the matrix class $\M\M^{*(n)}$ defined in \cref{sec:theory}. Note that a matrix in $\M\M^{*(b,n)}$ or $\M^*\M\ind{b,n}$. is represented by $2\frac{n^2}{b} + 2nb$ parameters.

Finally, we define the following ``Monarch hierarchy'' based on the kaleidoscope hierarchy of \cite{dao2020kaleidoscope}:
\begin{definition}[Class $(\M\M^{*(b,n)})^w_e$]
\label{def:block_MM}
Let $b \in (1, n)$ be an integer that divides $n$. We define the matrix class $(\M\M^{*(b,n)})^w_e$ as the set of all matrices $\vM$ that can be expressed as
    \begin{equation}
        	\label{eq:mm-hierarchy}
    \vM= \lt \pd{i=1}{w} \vM_i\rt [1:n, 1:n]
    \end{equation}
    where each $\vM_i \in \M\M^{*(b,e\cdot n)}$.
\end{definition}
Note that a matrix in $(\M\M^{*(b,n)})^w_e$ is represented by $2w\frac{e^2n^2}{b} + 2wenb$ parameters.

\subsubsection{Properties}
\label{sec:sq-Monarch-properties}
Here we show some properties of the matrix classes defined above. We first show some basic equivalent ways to define these classes. We then show (\cref{thm:lr_permutation}) that the matrices in $\DB\ind{b, n}$ are permuted block-diagonal matrices; specifically, that they can be converted to matrices in $\BD\ind{\frac{n}{b}, n}$ by applying the appropriate permutation. Finally, we state an expressivity result for the general ``Monarch hierarchy''  which follows from Theorem 1 of \cite{dao2020kaleidoscope}.

First, we define a class of permutations.
Let $1\le b\le n$ be integers such that $b$ divides $n$.
We will need to express each index $0\le i<n$ in ``block form.'' More specifically:

\begin{definition}\label{def:$i$}
Let $i \ge 0$, $b \ge 1$ be integers. Then define
\[i_0=i\mod{b},\]
and
\[i_1=\floors{\frac ib}.\] 
We use the notation $i\equiv\baseb{i_1}{i_0}{b}$ to denote the representation above. In particular, if $i\equiv(i_1,i_0)_{b}$,
then we have
\[
        i = i_1 \cdot b + i_0
\]
\end{definition}

Using this notation, we define the following class of permutations:
\begin{definition}
\label{def:sigma-b}
Let $b \in [1, n]$ be an integer that divides $n$.  Let $i\equiv\baseb{i_1}{i_0}{b}$. Define
    \begin{equation}
            \label{eq:sigma_b-def}
        \sigma_{(b,n)}(i) = i_0\cdot\frac{n}{b} + i_1.
    \end{equation}
That is, $\sigma_{(b,n)}(i)\equiv \baseb{i_0}{i_1}{\frac {n}{b}}$.
Let $\vP_{(b,n)}$ denote the $n \times n$ permutation matrix defined by the permutation $\sigma_{(b,n)}$.
\end{definition}
Intuitively, $\vP_{(b,n)}$ can be interpreted as reshaping a length-$n$ vector into an $b \times \ff{n}{b}$ matrix in row-major order, transposing the result, and then flattening this back into a vector (again in row-major order).

Now, we restate the formulation in \cref{def:square_r} equivalently as:
\begin{proposition}

\label{prop:R-eqv-def}
A matrix $\vR$ satisfies~\Cref{eq:def-R} (i.e., $\vR \in \BD\ind{b,n}$) if and only if the following holds for any
$0\le i,j< n$. Let $i\equiv\baseb{i_1}{i_0}{b}$ and $j\equiv\baseb{j_1}{j_0}{b}$.  Then
\begin{enumerate}
    \item\label{item:zero-loc-R} if $i_1\ne j_1$, then $\vR[i,j]=0$. 
    \item \label{item:nonzero-loc-R} Else (i.e., when $i_1=j_1$), then $\vR[i,j]=\vR_{i_1}[i_0,j_0]$.
\end{enumerate}

\end{proposition}

We restate the formulation in \cref{def:Matrix L} equivalently as:
\begin{proposition}
\label{prop:L-eqv-def}
A matrix $\vL$ satisfies~\Cref{eq:def-L} (i.e., $\vL \in \DB\ind{b,n}$) if and only if the following holds for any
$0\le i,j< n$. Let $i\equiv\baseb{i_1}{i_0}{b}$ and $j\equiv\baseb{j_1}{j_0}{b}$. Then 
\begin{enumerate}
    \item\label{item:zero-loc-L} if $i_0\ne j_0$, then $\vL[i,j]=0$. 
    \item \label{item:nonzero-loc-L} Else, (i.e., when $i_0=j_0$), then $\vL[i,j]=\vD_{i_1,j_1}[i_0,i_0]$.
\end{enumerate}
\end{proposition}

We will argue the following:
\begin{theorem}\label{thm:lr_permutation} Let $1\le b\le n$ such that $b$ divides $n$.
Recall that $\vP_{(b,n)}$ is the permutation matrix defined by the permutation $\sigma_{(b,n)}$. Let $\vL$ be a matrix in $\DB\ind{b, n}$. Then we have
\[\vR'=\vP_{(b,n)}\cdot\vL\cdot\vP_{(b,n)}^\top,\]
where $\vR' \in \BD\ind{\frac{n}{b},\, n}$.
\end{theorem}

\begin{proof}
We first note that multiplying an $n\times n$ matrix on the right (and left resp.) by $\vP_{(b,n)}^\top = \vP_{(\frac nb,n)}$ (and $\vP_{(b,n)}$ resp.) permutes the columns (and rows resp.) of the matrix according to $\sigma_{(b,n)}$.\footnote{This uses the fact that $\parens{\sigma_{(b,n)}}^{-1}=\sigma_{(\frac nb,n)}$ (which means $P_{(\frac{n}{b}, n)} = P_{(b, n)}^\top$ since the inverse of a permutation matrix is its transpose).} This implies that for any $0\le i,j<n$:
\begin{equation}
\label{eq:L-permuted}
\vR'[\sigma_{(b,n)}(i),\sigma_{(b,n)}(j)]=\vL[i,j].
\end{equation}
To complete the proof, we will argue that $\vR'$ satisfies the two conditions in~\Cref{prop:R-eqv-def}.

Towards this end, let $0\le i,j<n$ be arbitrary indices and further, define $i=\baseb{i_1}{i_0}{b}$ and $j=\baseb{j_1}{j_0}{b}$. Then note that $\sigma_{(b,n)}(i)=\baseb{i_0}{i_1}{\frac nb}$ and $\sigma_{(b,n)}(j)=\baseb{j_0}{j_1}{\frac nb}$.

By~\Cref{prop:L-eqv-def}, we have that if $i_0\ne j_0$, then $\vL[i,j]=0$. Note that $i_0\ne j_0$ satisfies the pre-condition for base size $\frac nb$ for indices $(\sigma_{(b,n)}(i),\sigma_{(b,n)}(j))$ in item~\ref{item:zero-loc-R} in~\Cref{prop:R-eqv-def}.  Then   by~\cref{eq:L-permuted}, we have that $\vR'[\sigma_{(b,n)}(i),\sigma_{(b,n)}(j)]=0$, which satisfies item~\ref{item:zero-loc-R} in~\Cref{prop:R-eqv-def}.

Now consider the case that $i_0=j_0$; then by item~\ref{item:nonzero-loc-L} in~\Cref{prop:L-eqv-def}, we have that $\vL[i,j]=\vD_{i_1,j_1}[i_0,i_0]$.  Note that $i_0= j_0$ satisfies the pre-condition for base size $\frac nb$ for indices $(\sigma_{(b,n)}(i),\sigma_{(b,n)}(j))$ in item~\ref{item:nonzero-loc-R} in~\Cref{prop:R-eqv-def} if we define $\vR'_{i_0}\in\F^{\frac nb\times\frac nb}$ as follows:
\[\vR'_{i_0}[i_1,j_1]=\vD_{i_1,j_1}[i_0,i_0].\] 
Note that the above implies that 
\[\vR'=\diag\parens{\vR'_0,\dots,\vR'_{b-1}},\]
where $\vR'_{\cdot}$ is as defined in the above paragraph. This means $\vR' \in \BD\ind{\frac{n}{b}, n}$, since each block $\vR_{i_0}'$ is a matrix of size $\frac{n}{b} \times \frac{n}{b}$.
\end{proof}

We now briefly note some alternate ways to express matrices in $\M\M^{*(b,n)}$.
\begin{proposition}
\label{prop:mm-eqv-def}
For any $\vM \in \M\M^{*(b,n)}$, we can write $\vM = (\vP_{(b,n)}^\top \vL_1 \vP_{(b,n)}) \vR (\vP_{(b,n)}^\top\vL_2\vP_{(b,n)})$, where $\vL_1,\vL_2 \in \BD\ind{\frac{n}{b},n}$ and $\vR \in \BD\ind{b,n}$.
\end{proposition}
\begin{proof}
By definition (see \cref{def:square_r} and \cref{def:Matrix L}), if $\vM \in \M\M^{*(b,n)}$, we can write
$\vM = (\vL_1' \vR_1) (\vL_2' \vR_2)^* = \vL_1' (\vR_1^* \vR_2) \vL_2'^*$,
where $\vL_1',\vL_2' \in \DB\ind{b,n},\vR_1,\vR_2 \in \BD\ind{b,n}$.

Notice that since $\vR_1^*, \vR_2$ are both block-diagonal with the same structure (i.e., both have blocks of size $b \times b$), their product $\vR$ is also in $\BD\ind{b,n}$.
Also, by \cref{thm:lr_permutation} we can write $\vL_1 = \vP_{(b,n)} \vL_1' \vP_{(b,n)}^\top$, $\vL_2 = \vP_{(b,n)} \vL_2' \vP_{(b,n)}^\top$, where $\vL_1,\vL_2$ are both in $\BD\ind{\frac{n}{b},n}$ (i.e., block diagonal with blocks of size $\frac{n}{b} \times \frac{n}{b}$).

Thus, we can write $\vM = (\vP_{(b,n)}^\top \vL_1 \vP_{(b,n)}) \vR (\vP_{(b,n)}^\top\vL_2\vP_{(b,n)})$, where $\vL_1,\vL_2 \in \BD\ind{\frac{n}{b},n}$ and $\vR \in \BD\ind{b,n}$.
\end{proof}

We use the above to show a simple relationship between $\M\M^{*(b,n)}$ and $\M^*\M^{(b,n)}$.
\begin{proposition}
\label{prop:mstarm}
If $\vM \in \M\M^{*(b,n)}$, then $\vP_{(b,n)} \vM \vP_{(b,n)}^\top \in \M^*\M\ind{\frac{n}{b},n}$. Conversely, if $\vM \in \M^*\M^{(b,n)}$, then $\vP_{(b,n)}^\top \vM \vP_{(b,n)} \in \M^*\M\ind{\frac{n}{b},n}$.
\end{proposition}
\begin{proof}
Suppose $\vM \in \M\M^{*(b,n)}$. By \cref{prop:mm-eqv-def} we can write $\vM = (\vP_{(b,n)}^\top \vL_1 \vP_{(b,n)}) \vR (\vP_{(b,n)}^\top\vL_2\vP_{(b,n)})$, where $\vL_1,\vL_2 \in \BD\ind{\frac{n}{b},n}$ and $\vR \in \BD\ind{b,n}$.
Thus $\vP_{(b,n)} \vM \vP_{(b,n)}^\top =
\vL_1 (\vP_{(b,n)} \vR \vP_{(b,n)}^\top) \vL_2$.

Letting $\vL_1' = \vL_1, \vL_2' = \vL_2^*, \vR_1' = \vP_{(b,n)} \vR \vP_{(b,n)}^\top$, and $\vR_2' = \vI$, we have $\vL_1', \vL_2' \in \BD\ind{\frac{n}{b}, n}$, $\vR_1', \vR_2' \in \DB\ind{\frac{n}{b}, n}$, and
$\vL_1 (\vP_{(b,n)} \vR \vP_{(b,n)}^\top) \vL_2 = 
\vL_1' \vR_1' \vR_2'^* \vL_2'^* = (\vL_1' \vR_1')(\vL_2' \vR_2')^* = \vM_1'\vM_2'^*$, where $\vM_1' = \vL_1' \vR_1', \vM_2' = \vL_2' \vR_2'$, so $\vM_1', \vM_2' \in \M^*\M\ind{\frac{n}{b},n}$.

Now instead suppose $\vM \in \M^*\M^{(b,n)}$. So $\vM = \vM_1^* \vM_2 = \vR_1^* \vL_1^* \vL_2 \vR_2$ for some $\vR_1, \vR_2 \in \BD\ind{b,n}$ and $\vL_1, \vL_2 \in \DB\ind{b,n}$. Thus by \cref{thm:lr_permutation} (and the fact that $\BD\ind{b,n}$ is closed under conjugate transposition) we can write $\vR_1^* = \vP_{(\frac{n}{b},n)}^\top \vR_1' \vP_{(\frac{n}{b}, n)} = \vP_{(b,n)} \vR_1' \vP_{(b, n)}^\top$ for some $\vR_1' \in \DB\ind{\frac{n}{b}, n}$, and similarly, can write $\vR_2 = \vP_{(b,n)} \vR_2' \vP_{(b,n)}^\top$ for some $\vR_2' \in \DB\ind{\frac{n}{b}, n}$. 

So $\vP_{(b,n)}^\top \vM \vP_{(b,n)} = \vR_1' (\vP_{(b, n)})^\top \vL_1^*)(\vL_2 \vP_{(b, n)})) \vR_2' =
 \vR_1' (\vP_{(b, n)}^\top \vL_1^* \vP_{(b, n)})(\vP_{(b, n)}^\top \vL_2 \vP_{(b, n)}) \vR_2' = (\vR_1' \vL_1')(\vL_2' \vR_2')$, where $\vL_1' = \vP_{(b, n)}^\top \vL_1^* \vP_{(b, n)}$, $\vL_2' = \vP_{(b, n)}^\top \vL_2 \vP_{(b, n)}$ are in $\BD\ind{\frac{n}{b}, n}$ by \cref{thm:lr_permutation}. Thus letting $\vM_1' = \vR_1'\vL_1'$, $\vM_2' = \vR_2^*\vL_2'^*$, we have $\vM = \vM_1' \vM_2'^*$ with $\vM_1', \vM_2' \in \M^{*(\frac{n}{b},n)}$.
\end{proof}

We now show that the class $\M\ind{b,n}$ strictly contains the class $\B\ind{n}$ of $n \times n$ butterfly matrices (as defined in \citet{dao2020kaleidoscope}). We first show two elementary ``helper'' results.

\begin{proposition}
\label{prop:bd-contain}
If $b,\, c \in (1, n)$ are such that $b$ divides $c$ and $c$ divides $n$, then $\BD\ind{b, n} \subseteq \BD\ind{c, n}$.
\end{proposition}
\begin{proof}
Suppose $\vR \in \BD\ind{b, n}$. Then by \cref{prop:R-eqv-def}, $\vR[i, j] = 0$ whenever $\floor{\frac{i}{b}} \ne \floor{\frac{j}{b}}$. Thus, whenever $\floor{\frac{i}{c}} \ne \floor{\frac{j}{c}}$, $\vR[i, j] = 0$, since $\floor{\frac{i}{c}} \ne \floor{\frac{j}{c}}$ implies $\floor{\frac{i}{b}} \ne \floor{\frac{j}{b}}$ by the assumption that $b$ divides $c$.
Applying \cref{prop:R-eqv-def} again, this means $\vR \in \BD\ind{c,n}$ as well.
\end{proof}

\begin{proposition}
\label{prop:db-contain}
If $b,\, c \in (1, n)$ are such that $b$ divides $c$ and $c$ divides $n$, then $\DB\ind{c, n} \subseteq \DB\ind{b, n}$.
\end{proposition}
\begin{proof}
Suppose $\vL \in \DB\ind{c, n}$. Then by \cref{prop:L-eqv-def}, $\vL[i, j] = 0$ whenever $(i \mod c) \ne (j \mod c)$. Thus, whenever $(i \mod b) \ne (j \mod b)$, $\vL[i, j] = 0$, since $(i \mod b) \ne (j \mod b)$ implies $(i \mod c) \ne (j \mod c)$ by the assumption that $b$ divides $c$.
Applying \cref{prop:L-eqv-def} again, this means $\vL \in \DB\ind{b,n}$ as well.
\end{proof}

\begin{theorem}
\label{thm:b_contained}
Let $n \ge 4$ be a power of 2. The class of matrices $\B\ind{n}$ is a subset of the class $\M\ind{b, n}$, for all $b \in (1, n)$ that divide $n$. When $n \ge 512$ it is a strict subset.
\end{theorem}
\begin{proof}
Recall from \cref{sec:butterfly} that if $\vB \in \B\ind{n}$, it has a \emph{butterfly factorization} 
$\vB = \vB_n \vB_{n/2} \hdots \vB_2$, where each $\vB_i \in \BF\ind{n, i}$.

Consider multiplying together the factors $\vB_b \vB_{b/2} \dots \vB_2$ (where $b \in (1, n)$ divides $n$). Since $\vB_i \in \BF\ind{n,i}$, by definition it is block diagonal with diagonal blocks of size $i \times i$; in other words, $\vB_i \in \BD\ind{i, n}$. Thus, each of the matrices $\vB_b, \vB_{b/2}, \dots, \vB_2$ is in $\BD\ind{b, n}$ (by \cref{prop:bd-contain}), i.e. block-diagonal with block size $b \times b$. This means their product $\vB_b \vB_{b/2} \dots \vB_2$ is also block diagonal with block size $b \times b$, i.e., it is in $\BD\ind{b, n}$.

Now, note that since $\vB_i \in \BF\ind{n,i}$, by definition it is a block matrix with blocks of size $i/2 \times i/2$, where each block is a diagonal matrix (note that some of these blocks are zero, except for the case of $\vB_n$). In other words, $\vB_i \in \DB\ind{i/2, n}$. Thus, for all $i \in \{n, n/2, \dots, 2b\}$, $\vB_i \in \DB\ind{(2b)/2, n} = \DB\ind{b, n}$ (by \cref{prop:db-contain}). So, their product $\vB_n \vB_{n/2} \dots \vB_{2b}$ is in $\DB\ind{b, n}$ as well, as by \cref{thm:lr_permutation} we can write $\vB_n \vB_{n/2} \dots \vB_{2b} = \vP_{(b,n)}^\top (\vP_{(b,n)} \vB_n \vP_{(b,n)}^\top) (\vP_{(b,n)} \vB_{n/2} \vP_{(b,n)}^\top) \dots (\vP_{(b,n)} \vB_{2b} \vP_{(b,n)}^\top) \vP_{(b,n)}$ and each of the $\vP_{(b,n)} \vB_i \vP_{(b,n)}^\top$'s in the preceding expression is in $\BD\ind{\frac{n}{b}, n}$.

Thus, if we let $\vL = \vB_n \vB_{n/2} \dots \vB_{2b}$ and $\vR = \vB_b \vB_{b/2} \dots \vB_2$,  we have $\vB = \vL\vR$ and $\vL \in \DB\ind{b, n}$, $\vR \in \BD\ind{b, n}$, which means that $\vB \in \M\ind{b,n}$ (\cref{def:block_Monarch}).

To show that the inclusion is strict, notice that any $\vM \in \M\ind{b,n}$ is the product of $\vL$ and $\vR$, where $\vR \in \BD\ind{b, n}$ and $\vP_{(b,n)}^\top \vL \vP_{(b,n)} \in \BD\ind{\frac{n}{b}, n}$ (by \cref{thm:lr_permutation}). Notice that the identity matrix is contained in both $\BD\ind{b,n}$ and $\DB\ind{b,n}$. Suppose first that $b \le \sqrt{n}$. Then even if we set $\vR$ to the identity, $\vM$ has at least $\frac{n^2}{b} \ge n^{3/2}$ free parameters (the entries in the blocks of the block-diagonal matrix $\vP_{(b,n)}^\top \vL \vP_{(b,n)}$ can be arbitrary, and there are $b$ such blocks each of size $\frac{n}{b}$). Similarly, in the case $b > \sqrt{n}$, we can set $\vL$ to the identity, and $\vM$ has at least $nb \ge n^{3/2}$ free parameters (the entries of the block-diagonal matrix $\vR$ can be arbitrary, and there are $nb$ total of these). Thus, at least $n^{3/2}$ parameters are required to uniquely describe any matrix in $\M\ind{b,n}$. However, a butterfly matrix in $\B\ind{n}$ has only $2n \log_2 n$ parameters. For $n > 256$, $2n \log_2 n < n^{3/2}$. (Note that this analysis is not tight: a more careful analysis can show the inclusion is strict even for smaller values of $n$.)

\end{proof}

We end this section with a theorem on the expressivity of the ``monarch hierarchy'' (products of monarch matrices), which follows from Theorem 1 of \cite{dao2020kaleidoscope}.
\begin{theorem}[Monarch hierarchy expressivity]
\label{thm:monarch_hierarchy}
Let $\vM$ be an $n \times n$ matrix such that matrix-vector multiplication of $\vM$ and an arbitrary vector $\vv$ (i.e., computation of $\vM \vv$) can
be represented as a linear arithmetic circuit with depth $d$ and $s$ total gates. Let $b \in (1, n)$ be a power of 2 that divides $n$.
Then, $\vM \in (\M\M^{*(b, n)})^{O(d)}_{O(s/n)}$.
\end{theorem}
\begin{proof}
Theorem 1 of \citet{dao2020kaleidoscope} says that if $n$ is a power of 2 and $\vA$ is an $n \times n$ matrix such that multiplying any vector $v$ by $\vA$ can be represented as a linear arithmetic circuit with depth $\le d$ and $\le s$ total gates, then $\vA \in (\B\B^{*(n)})^{O(d)}_{O(s/n)}$ (this is the ``kaleidoscope representation'' of $\vA$).

Recall from \cref{thm:b_contained} that for any $b \in (1, n)$ that is a power of 2 and divides $n$, $\M\ind{b, n} \supset \B\ind{n}$; thus, this implies $\M\M^{*(b,e\cdot n)} \supset \B\B^{*(e\cdot n)}$, and in turn $(\M\M^{*(b,n)})^w_e \supset (\B\B^{*(n)})^w_e$.

As $\vA \in  (\B\B^{*(n)})^{O(d)}_{O(s/n)}$, we thus have $\vA \in  (\M\M^{*(b,n)})^{O(d)}_{O(s/n)}$.
\end{proof}

As per \cite{dao2020kaleidoscope}, the class of kaleidoscope matrices $(\B\B^{*(n)})^{O(d)}_{O(s/n)}$ has $O(ds \log s)$ parameters and runtime, compared to the $O(s)$ parameters and runtime of the circuit. Note that at worst, $s$ is $O(n^2)$.

Define $f(n,s)$ to be the largest power of 2 that is $\le \min\left\{\ff{n}{2}, \sqrt{s}\right\}$. Note that $f(n,s) = O(\sqrt{s})$, and since $s = O(n^2)$, $f(n,s) = \Omega(\sqrt{s})$, so $f(n,s) = \Theta(\sqrt{s})$. %
We thus have $\vA \in (\M\M^{*(f(n,s), n)})^{O(d)}_{O(s/n)}$. The class $(\M\M^{*(f(n,s), n)})^{O(d)}_{O(s/n)}$ has $O(d\frac{s^2}{f(n,s)} + dsf(n,s)) = O(ds^{3/2})$ parameters. Thus, the monarch representation of $\vA$ is suboptimal by at most an $O(d\sqrt{s})$ factor compared to the $O(d{}\,\log s)$ of kaleidoscope.

\subsection{General rectangular matrices}
\label{sec:Monarch_rect}
In this section, we extend the Monarch parametrization to apply to \emph{rectangular} matrices, and prove some basic properties of the relevant matrix classes. (Note that our subsequent theoretical results (\cref{sec:proofs}) do not depend on this section, as they focus on the square parametrization.)

For the rest of the section, we will assume that $n_1, n_2, n_3, b_1, b_2 , b_3 \ge 1$ are integers such that:
\begin{itemize}
\item $b_i$ divides $n_i$ for all $1\le i\le 3$, and 
\item $\frac{n_1}{b_1} = \frac{n_2}{b_2}$.
\end{itemize}

We begin with the definition of the following class of rectangular block-diagonal matrices:
\begin{definition}
\label{def:monarch_rectangular}
For $0\le i< \frac{n}{b_1}$, let $\vR_{i}\in\F^{b_2 \times b_1}$ be a $b_2 \times b_1$ matrix. Then define the matrix $\vR \in \F^{n_2\times n_1}$ as follows:
\begin{equation}
 \label{eq:def-rect-R}
  \vR = \diag\left(\vR_0, \dots, \vR_{\frac {n_1}{b_1}-1}\right).
\end{equation}
\end{definition}
We say that $\vR$ has {\em block size} $b_2 \times b_1$. Recall that we have assumed $\frac {n_1}{b_1}=\frac{n_2}{b_2}$, so~\cref{eq:def-rect-R} is well-defined.
(Note that the number of possible nonzero values in $\vR$ is $\frac {n_1}{b_1}\cdot b_1 \times b_2 =n_1b_2$.)
We denote the class of all matrices $\vR$ expressible in this form by $\BD\ind{b_2 \times b_1, n_2 \times n_1}$.
Note that this class is only defined when $\frac {n_1}{b_1}=\frac{n_2}{n_2}$.

We restate the above definition equivalently as:
\begin{proposition} \label{prop:rect-R-eqv-def} 
$\vR\in\F^{n_2\times n_1}$ is in $\BD\ind{b_2 \times b_1, n_2 \times n_1}$ (with $\frac {n_1}{b_1}=\frac{n_2}{n_2}$) if and only if the following holds for any
$0\le i < n_2$ and $0\le j< n_1$. Let $i\equiv\baseb{i_1}{i_0}{b_2}$ and $j\equiv\baseb{j_1}{j_0}{b_1}$ (recalling this notation from \cref{def:$i$}.  Then
\begin{enumerate}
    \item\label{item:rect-zero-loc-R} if $i_1\ne j_1$, then $\vR[i,j]=0$. 
    \item \label{item:rect-nonzero-loc-R} Else (i.e., when $i_1=j_1$), then $\vR[i,j]=\vR_{i_1}[i_0,j_0]$.
\end{enumerate}

\end{proposition}

Before we define the rectangular $\vL$, we first need to define the notion of a `wrapped diagonal' matrix:
\begin{definition} 
\label{def:wrapped-diag}
A {\em wrapped diagonal} matrix $\mx{S} \in\F^{b_3\times b_2}$ is defined as follows. First assume $b_2\le b_3$. Then for any $0\le i<b_3$ and $0\le j<b_2$, we have the following. If $i\mod{b_2}\ne j$, then $\vS[i,j]=0$. (If $b_2>b_3$, then instead apply the previous definition to $\vS^{\top}$.)

\end{definition}

We now define the following class of block matrices with each block a \emph{wrapped diagonal} matrix.
\begin{definition}\label{def:rect-Matrix L}
Let $\vL\in\F^{n_3\times n_2}$ have the form: 
    \begin{equation}
        	\label{eq:rect-def-L}
    \vL=
    	\begin{bmatrix}
    		\mx{S}_{0,0} & \dots & \mx{S}_{0,\frac{n_2}{b_2} -1} \\
    		\vdots & \ddots & \vdots \\
    		\mx{S}_{\frac{n_3}{b_3} -1,0} & \dots & \mx{S}_{\frac{n_3}{b_3} -1,\frac{n_2}{b_2} -1}
    	\end{bmatrix},
    \end{equation}
where each $\vS_{\cdot,\cdot}$ is a wrapped diagonal matrix in $\F^{b_3 \times b_2}$.
\end{definition}
We say that $\vL$ has {\em block size} $b_3 \times b_2$.
(Note that the number of possible nonzero values in $\vL$ is $\parens{\frac {n_2}{b_2}\cdot \frac{n_3}{b_3}} \max(b_2,b_3)=\frac{n_2 \cdot n_3}{\min(b_2,b_3)}$.)
We denote the class of all matrices $\vL$ expressible in this form by $\DB\ind{b_3 \times b_2, n_3 \times n_2}$.

We restate the above definition equivalently as:

\begin{proposition}
\label{prop:rect-L-eqv-def}
$\vL\in\F^{n_3\times n_2}$ is in $\DB\ind{b_3 \times b_2, n_3 \times n_2}$ if and only if the following holds for any
$0\le i < n_3$ and $0 \le j< n_2$. Let $i\equiv\baseb{i_1}{i_0}{b_3}$ and $j\equiv\baseb{j_1}{j_0}{b_2}$.  Assuming $b_2 \le b_3$, we have:
\begin{enumerate}
    \item\label{item:rect-zero-loc-L} if $i_0\mod{b_2}\ne j_0$, then $\vL[i,j]=0$. 
    \item \label{item:rect-nonzero-loc-L} Else, (i.e., when $i_0\mod{b_2}=j_0$), then $\vL[i,j]=\vS_{i_1,j_1}[i_0,j_0]$.
\end{enumerate}
If $b_2>b_3$, then in the above, the condition ``$i_0\mod{b_2}\ne j_0$'' gets replaced by ``$j_0\mod{b_2}\ne i_0$.''
\end{proposition}

Using the above definitions, we now define the class of rectangular Monarch matrices.
\begin{definition}[Rectangular Monarch Matrix]
\label{def:block_Monarch}
Let $\vM \in \F^{n_3 \times n_1}$ be a matrix of the form: 
    \begin{equation}
        	\label{eq:Monarch-general}
    \vM= \vL \vR
    \end{equation}
    where $\vL \in \DB\ind{b_3 \times b_2, n_3 \times n_2}$ and $\vR \in \BD\ind{b_2 \times b_1, n_2 \times n_1}$.
\end{definition}
(As mentioned before, we assume $b_i$ divides $n_i$ for $i = 1,2,3$ and that $n_1/b_1 = n_2/b_2$.)
We denote the class of all matrices $\vM$ expressible in this form by $\M\ind{(b_1,b_2,b_3), (n_1,n_2,n_3)}$. Observe that when $b_1 = b_2 = b_3 = b$ and $n_1 = n_2 = n_3 = n$, this is exactly the matrix class $\M\ind{b, n}$ in \cref{def:block_Monarch}.

We are now ready to prove our main result in this section, which essentially follows from the observation that if we permute the rows and columns of $\vL$ such that the row/column block size in $\vL$ becomes the number of row/columns blocks in the permuted matrix (and vice-versa) then the permuted matrix has the form of $\vR$.

\begin{theorem} Let $1\le b,n_2,n_3$ be such that $b$ divides $n_2$ and $n_3$.
Suppose $\vL\in\F^{n_3\times n_2} \in \DB\ind{b \times b, n_3 \times n_2}$.
Then if we define
\[\vR'=\vP_{(b,n_3)}\cdot\vL\cdot\vP_{(b,n_2)}^\top,\]
we have that $\vR' \in \BD\ind{\frac{n_3}{b_3} \times \frac{n_2}{b_2}, n_3 \times n_2}$.
\end{theorem}

\begin{proof}
We recall  that multiplying an $m\times n$ matrix on the right (and left resp.) by $\vP_{(b,n)}^\top = \vP_{(\frac nb,n)}$ (and $\vP_{(b,m)}$ resp.) permutes the columns (and rows resp.) of the matrix according to $\sigma_{(b,n)}$ (and $\sigma_{(b,m)}$) respectively.\footnote{This uses the fact that $\parens{\sigma_{(b,n)}}^{-1}=\sigma_{(\frac nb,n)}$.} This implies that for any $0\le i,j<n$:
\begin{equation}
\label{eq:rect-L-permuted}
\vR'[\sigma_{(b,n_3)}(i),\sigma_{(b,n_2)}(j)]=\vL[i,j].
\end{equation}

Recall that in the notation of \cref{def:rect-Matrix L} we have $b_2=b_3=b$, so we are in the $b_2 \le b_3$ case.
To complete the proof, we will argue that $\vR'$ satisfies the two conditions in~\Cref{prop:rect-R-eqv-def}.\footnote{Note that we also need that the ratios of the row/column length to the row/column block sizes are the same; i.e., in our case we need that $\frac{n_3}{n_3 / b_3}=\frac{n_2}{n_2 / b_2}$, which is true because $b_2=b_3=b$.}

Towards this end, let $0\le i,j<n$ be arbitrary indices and further, define $i=\baseb{i_1}{i_0}{b}$ and $j=\baseb{j_1}{j_0}{b}$. Then note that $\sigma_{(b,n_3)}(i)=\baseb{i_0}{i_1}{\frac {n_3}{b}}$ and $\sigma_{(b,n_2)}(j)=\baseb{j_0}{j_1}{\frac {n_2}{b}}$.

By~\Cref{prop:rect-L-eqv-def}, we have that if $i_0\mod{b}\ne j_0$, then $\vL[i,j]=0$. Note  that since $i_0,j_0<b$ by definition, the condition $i_0\mod{b}\ne j_0$ is equivalent to saying $i_0\ne j_0$. Note that $i_0\ne j_0$ satisfies the pre-condition for base size $\frac {n_3}{b}\times \frac{n_2}{b}$ for indices $(\sigma_{(b,n_3)}(i),\sigma_{(b,n_2)}(j))$ in item~\ref{item:rect-zero-loc-R} in~\Cref{prop:rect-R-eqv-def}.  Then   by~\cref{eq:rect-L-permuted}, we have that $\vR'[\sigma_{(b,n_3)}(i),\sigma_{(b,n_2)}(j)]=0$, which satisfies item~\ref{item:rect-zero-loc-R} in~\Cref{prop:rect-R-eqv-def}.

Now consider the case that $i_0=j\mod b$, which by the observation in the above paragraph is the same as $i_0=j_0$. Then by item~\ref{item:rect-nonzero-loc-L} in~\Cref{prop:rect-L-eqv-def}, we have that $\vL[i,j]=\vS_{i_1,j_1}[i_0,j_0]$.  Note that $i_0= j_0$ satisfies the pre-condition for base size $\frac{n_3}{b}\times \frac{n_2}{b}$ for indices $(\sigma_{(b,n_3)}(i),\sigma_{(b,n_2)}(j))$ in item~\ref{item:rect-nonzero-loc-R} in~\Cref{prop:rect-R-eqv-def} if we define $\vR'_{i_0}\in\F^{\frac{n_3}{b}\times\frac{n_2}{b}}$ as follows:
\[\vR'_{i_0}[i_1,j_1]=\vS_{i_1,j_1}[i_0,j_0].\] 

Note that the above implies that 
\[\vR'=\diag\parens{\vR'_0,\dots,\vR'_{b-1}},\]
where $\vR'_{\cdot}$ is as defined in the above paragraph.
This means $\vR' \in \BD\ind{\frac{n_3}{b} \times \frac{n_2}{b}, n_3 \times n_2}$, since $\vR'$ has size $n_3 \times n_2$ and each block $\vR_{i_0}'$ is a matrix of size $\frac{n_3}{b} \times \frac{n_2}{b}$.
\end{proof}

\section{Theory}
\label{sec:proofs}

\subsection{Expressiveness of $\M$}
\label{subsec:expressiveness_proof}

\begin{proof}[Proof of \cref{thm:Monarch_expressiveness}]
  As~\citet[Appendix J]{dao2020kaleidoscope} show, the matrix class $\BBS$ can
  represent convolution, Hadamard transform, Toeplitz matrices, and AFDF.
  Since the Monarch class $\MMS$ contains the butterfly class $\BBS$ (which follows from \cref{thm:b_contained}), it follows
  that $\MMS$ can also represent those transforms / matrices.

  Note that the Hadamard transform is actually in $\B$ \citep{dao2020kaleidoscope}, so it is in $\M$ as well.

  \citet[Appendix J]{dao2020kaleidoscope} also show that the matrix class $(\BBS)^2$ can
  represent the Fourier, discrete sine/cosine transforms, the $(HD)^3$ class,
  Fastfood, and ACDC matrices.
  By the same argument, as the Monarch class $(\MMS)^2$ contains the butterfly
  class $(\BBS)^2$, $(\MMS)^2$ can thus also represent these transforms / matrices.
\end{proof}

\subsection{Projection onto $\M$}

In \cref{alg:project}, we provide pseudocode for the algorithm outlined in
\cref{subsec:projection}.
We now prove \cref{thm:Monarch_projection}.
Note that the rectangular matrix case generalizes naturally from the square matrix case, by replacing square blocks with rectangular blocks.

\begin{proof}[Proof of \cref{thm:Monarch_projection}]
  As shown in \cref{subsec:projection}, after reshaping the Monarch matrix $\vM$ as a 4D tensor $M_{\ell  jki}$ and writing the two block-diagonal matrices $\vL$ and $\vR$ as 3D tensors $L_{j \ell  k}$ and $R_{k j i}$, we obtain:
  \begin{equation*}
    M_{\ell  j k i} = L_{j \ell  k} R_{k j i}, \quad \text{for } \ell, j, k, i = 1, \dots, m.
  \end{equation*}
  We can similarly reshape the given matrix $A$ into a 4D tensor $A_{\ell j k i}$ with size $m \times m \times m \times m$.

  Since the squared Frobenius norm objective $\norm{A - M}_F^2$ (\cref{eq:projection_objective}) only depends on the entries of $A$ and $M$ and not their shape,
  we can rewrite the objective after reshaping:
  \begin{align*}
    \norm{A - M}_F^2
    &= \sum_{\ell  j k i} \left(A_{\ell  j k i} - M_{\ell  j k i}\right)^2 \\
    &= \sum_{\ell  j k i} \left( A_{\ell  j k i} - L_{j \ell  k} R_{k j i} \right)^2 \\
    &= \sum_{j k} \sum_{\ell i} \left( A_{\ell  j k i} - L_{j \ell  k} R_{k j i} \right)^2.
  \end{align*}
  We see that the objective decomposes into $m \times m$ independent terms (indexed by $j$ and $k$).
  For each value of $j$ and $k$, the objective is exactly the rank-1 approximation objective for the corresponding slice $\vA_{:, j, k, :}$.

  Let $\vu_{jk} \vv_{jk}^\top$ be the best rank-1 approximation of $\vA_{:, j, k, :}$ (which we can compute using the SVD, by the Eckart--Young theorem~\citep{eckart1936approximation} for Frobenius norm).
  Let $\vR$ be the 3D tensor of size $m \times m \times m$ where $\vR_{kji} = (\vv_{jk})_i$, and let $\vL$ be the 3D tensor of size $m \times m \times m$ where $\vL_{j \ell k} = (\vu_{jk})_\ell$.
  Then each of the terms in the objective is minimized, and thus the overall objective is minimized.

  We see that the algorithm requires $m \cdot m$ SVD's, each of size $m \times m$.
  Each SVD takes $O(m^3)$ time~\citep{trefethen2000spectral}, so the overall time complexity is $O(m^5) = O(n^{5/2})$.
\end{proof}

\subsection{Monarch Factorizations for Matrices in $\MMS$}
In this section, we describe the algorithm for factorizing matrices in $\MMS$ previously outlined in \cref{subsec:recovery} (\cref{alg:mm_recovery}). Again, \cref{alg:mm_recovery} handles the general case where the block sizes of $L$ and $R$ can be different. We then prove \cref{thm:Monarch_recovery_full}, which has \cref{thm:Monarch_recovery} as an immediate corollary.

Our goal is thus to compute the matrices $\vL_1,\vR,\vL_2$ in the factorization of $\vM$.
In order to compute this factorization, we require the following assumption on $\vM$:
\begin{assumption}
\label{assump:a1_appendix}
Assume that (1) $\vM \in \M\M^{*(b,n)}$ is invertible and (2) $\vM$ can be written as $(\vP_{(b,n)}^\top \vL_1 \vP_{(b,n)}) \vR (\vP_{(b,n)}^\top\vL_2\vP_{(b,n)})$, where $\vL_1,\vL_2 \in \BD\ind{\frac{n}{b},n},\vR \in \BD\ind{b,n}$, and $\vR$ has no nonzero entries in its diagonal blocks.
(Note that by \cref{prop:mm-eqv-def}, we can write any $\vM \in \M\M^{*(b,n)}$ as $(\vP_{(b,n)}^\top \vL_1 \vP_{(b,n)}) \vR (\vP_{(b,n)}^\top\vL_2\vP_{(b,n)})$; thus, (2) is merely the assumption that $\vR$ has no zero entries in its blocks.)
\end{assumption}%

This is analogous to \cref{assump:a1}, except applicable to the more general block size $b$.
We now present \cref{alg:mm_recovery} to find factors $\vL_1,\vR,\vL_2$ of matrices satisfying \cref{assump:a1_appendix}.

First, observe that if we define $\MM = \vP_{(b,n)} \vM \vP_{(b,n)}^\top$, we have $\MM = \vL_1 (\vP_{(b,n)} \vR \vP_{(b,n)}^\top) \vL_2$. By \cref{thm:lr_permutation}, the matrix $\vP_{(b,n)} \vR \vP_{(b,n)}^\top$ is in $\DB\ind{\frac{n}{b}, n}$, i.e., is a block matrix with blocks of size $\frac{n}{b} \times \frac{n}{b}$ where each block is a diagonal matrix.
Thus, we can write:

\begin{equation*}
\lt \begin{array}{cccc} \MM_{11} & \MM_{12} & \dots & \MM_{1b} \\ \MM_{21} & \MM_{22} & \dots & \MM_{2b} \\ \ddots & \ddots & \ddots & \ddots  \\ \MM_{b1} & \MM_{b2} & \dots & \MM_{bb} \end{array}\rt
= \lt \begin{array}{ccccc} \vA_1 \\ & \vA_2 \\ & & \ddots \\ & & & \vA_{b} \end{array}\rt
\lt \begin{array}{cccc} \vD_{11} & \vD_{12} & \dots & \vD_{1b} \\ \vD_{21} & \vD_{22} & \dots & \vD_{2b} \\ \ddots & \ddots & \ddots & \ddots  \\ \vD_{b1} & \vD_{b2} & \dots & \vD_{bb} \end{array}\rt
\lt \begin{array}{ccccc} \vC_1 \\ & \vC_2 \\ & & \ddots \\ & & & \vC_{b} \end{array}\rt,
\end{equation*}

where $\vA_1,\dots,\vA_b$ are $\frac{n}{b} \times \frac{n}{b}$ matrices that are the diagonal blocks of $\vL_1$; $\vC_1,\dots,\vC_b$ are $\frac{n}{b} \times \frac{n}{b}$ matrices that are the diagonal blocks of $\vL_2$; $\vD_{11},\dots,\vD_{1b},\vD_{21},\dots,\vD_{2b},\dots,\vD_{b1},\dots,\vD_{bb}$ are $\frac{n}{b} \times \frac{n}{b}$ \emph{diagonal} matrices that are the blocks of $\vP_{(b,n)} \vR \vP_{(b,n)}^\top$; and
$\MM_{11},\dots,\MM_{1b},\MM_{21},\dots,\MM_{2b},\dots,\MM_{b1},\dots,\MM_{bb}$ are $\frac{n}{b} \times \frac{n}{b}$ matrices that are the blocks of $\MM = \vP_{(b,n)} \vM \vP_{(b,n)}^\top$.

Thus, we have the set of matrix equations $\vA_i \vD_{ij} \vC_j = \MM_{ij}$, for $1 \le i, j \le b$. Notice that the assumption that the $\vR$ has no nonzero entries in its blocks (\cref{assump:a1_appendix}) is equivalent to assuming that none of the diagonal entries of any matrix $\vD_{ij}$ is equal to zero. Also, the assumption that $\vM$ is invertible implies that $\vL_1, \vL_2$ are invertible (since the product of square singular matrices is singular), which in turn implies that each block matrix $\vA_i$ and each block matrix $\vC_j$ is invertible (since a square block-diagonal matrix where one of the blocks is singular is itself singular). Taken together, this means that each matrix $\MM_{ij}$ is invertible, since $\MM_{ij} = \vA_i \vD_{ij} \vC_j$ and each of the matrices on the RHS of the equation is invertible.

Observe that given a solution to the set of equations $\vA_i \vD_{ij} \vC_j = \MM_{ij}$, if we rescale and permute the matrices $\vA_i, \vD_{ij}, \vC_j$ appropriately, the result is still a solution to the equations. Specifically, let $\vP$ be any permutation matrix and $\{\vS_i\}_{i=1}^b, \{\vS_j'\}_{j=1}^b$ be any invertible diagonal matrices (i.e., diagonal matrices without any zeros on the diagonal). 
Define $\vD_{ij}' = \vS_i \vP^\top {\vD}_{ij} \vP \vS_j'$ for all $i, j$. Notice that $\vP^\top \vD_{ij} \vP = \vP^{-1} \vD_{ij} \vP$ is diagonal because $\vD_{ij}$ is diagonal. Thus, $\vD_{ij}'$ is diagonal (and invertible) since the product of diagonal matrices is diagonal. Define $\vA_i' = \vA_i \vP \vS_i^{-1}$ and $\vC_j' = \vP^\top \vS_j'^{-1} \vC_j$ for all $i, j$.
Thus, we have that $\MM_{ij} = \vA_i \vD_{ij} \vC_j = (\vA_i \vP \vS_i^{-1} ) \vD_{ij}' (\vP^\top \vS_j'^{-1} \vC_j) = {\vA}_i' \vD_{ij}' \vC_j'$ for all $i, j$: in other words, we can scale the $\vA_i$'s on the right by any invertible diagonal matrix, the $\vC_j$'s on the left by any invertible diagonal matrix, and apply a matching permutation to the rows of the $\vC_j$'s and the columns of the $\vA_i$'s, and apply matching transformations to the $\vD_{ij}$'s and the result will still be a valid factorization. This implies that as long as we recover a ``correct'' $\hat{\vC}_1$ up to a permutation and scaling of its rows, we can set the $\hat{\vD}_{i1}$'s and $\hat{\vD}_{1j}$'s to the identity matrix, and then compute the remaining $\hat{\vA}_i$'s and $\hat{\vC}_j$'s via the equations $\hat{\vA}_i = \MM_{i1}\hat{\vC}_1^{-1}$ and $\hat{\vC}_j = \hat{\vA}_1^{-1}\MM_{1j}$.

To understand how we can compute such a matrix $\hat{\vC}_1$, define $\vF(i, j) = \MM_{i1}^{-1} \MM_{ij} \MM_{1j}^{-1} \MM_{11}$ and observe that
\begin{align*}
\vF(i, j) &= \MM_{i1}^{-1} \MM_{ij} \MM_{1j}^{-1} \MM_{11} \\ &=
(\vC_1^{-1} \vD_{i1}^{-1} \vA_i^{-1}) (\vA_i \vD_{ij} \vC_j) (\vC_j^{-1} \vD_{1j}^{-1} \vA_1^{-1}) (\vA_1 \vD_{11} \vC_1) \\
&= \vC_1^{-1} (\vD_{i1}^{-1} \vD_{ij} \vD_{1j}^{-1} \vD_{11}) \vC_1
\end{align*} for all $1 \le i, j \le b$.
Note that $\vD_{i1}^{-1} \vD_{ij} \vD_{1j}^{-1} \vD_{11}$ is a diagonal matrix;
thus, $\vC_1 \vF(i, j) \vC_1^{-1}$ is diagonal for all $i, j$, i.e., 
$\vC_1$ simultaneously diagonalizes all the matrices $\vF(i, j)$.
(Note: In this paper, we say that a matrix $\vQ$ ``simultaneously diagonalizes'' a set of matrices $\vG_1, \dots, \vG_k$ if $\vQ \vG_i \vQ^{-1}$ is a diagonal matrix for all $1 \le i \le k$. Note that sometimes the opposite convention [i.e., $\vQ^{-1} \vG_i \vQ$ must be diagonal] is used in the literature; we adopt the former for notational convenience.)
Indeed, if \emph{any} matrix simultaneously diagonalizes all these matrices, then it leads to a valid factorization, which we show in the proof of \cref{thm:Monarch_recovery_full}. Therefore, we compute some matrix that simultaneously diagonalizes all these matrices, and set $\hat{\vC}_1$ to that matrix.

These ideas form the basis of \cref{alg:mm_recovery}, which is presented formally below. \cref{alg:mm_recovery} uses simultaneous diagonalization as a subroutine; we discuss how to solve simultaneous diagonalization problems below.

\begin{algorithm}[H]
\caption{$\MMS$ Factorization}
\label{alg:mm_recovery}
\begin{algorithmic}[1]
\REQUIRE Block size $b$; matrix $\vM \in \M\M^{*(b,n)}$ satisfying \cref{assump:a1_appendix}
\item Define $\MM_{ij}$ (of size $\frac{n}{b} \times \frac{n}{b}$) as the $i,j$ block of $\vP_{(b,n)} \vM \vP_{(b,n)}^\top$
\FOR{$1 \le i, j \le b$}
\item Compute $\vF(i,j) := \MM_{i1}^{-1}\MM_{ij}\MM_{1j}^{-1}\MM_{11}$
\ENDFOR
\item $\hat{\vC}_1 \leftarrow \ \textsc{SIMULTANEOUS\_DIAG}\lt \{\vF(i,j)\}_{i,j=1,1}^{b,b}\rt $
\FOR{$1 \le i \le b$}
\item $\hat{\vA}_i \leftarrow \MM_{i1} \hat{\vC}_1^{-1}$
\ENDFOR
\FOR{$2 \le j \le b$}
\item $\hat{\vC}_j \leftarrow \hat{\vA}_1^{-1} \MM_{1j}$
\ENDFOR
\FOR{$1 \le i, j \le b$}
\item $\hat{\vD}_{ij} \leftarrow \hat{\vA}_i^{-1} \MM_{ij}\hat{\vC}_j^{-1}$
\ENDFOR
\end{algorithmic}
\end{algorithm}

\begin{theorem}\label{thm:Monarch_recovery_full}
Given an $n \times n$ matrix $\vM \in \M\M^{*(b,n)}$ satisfying Assumption \ref{assump:a1}, \cref{alg:mm_recovery} finds its Monarch factors $\vL_1, \vR, \vL_2$ in  time $O\lt \frac{n^3}{b} \rt$.
\end{theorem}

Notice that by setting $b = \sqrt{n}$, we immediately recover \cref{thm:Monarch_recovery}.
Note also that by \cref{prop:mstarm}, \cref{thm:Monarch_recovery_full} implies that given an $\vM \in \M^*\M^{(\frac{n}{b},n)}$, we can find its Monarch factorization in time $O(\frac{n^3}{b})$ as well (e.g., simply permute it to a matrix in $\M\M^{*(b,n)}$ and then run \cref{alg:mm_recovery}). 
We now prove \cref{thm:Monarch_recovery_full}.

\begin{proof}
We first show that the factorization returned by \cref{alg:mm_recovery} is valid, which reduces to showing that (1) $\MM_{ij} = \hat{\vA}_i \hat{\vD}_{ij} \hat{\vC}_j$ and (2) $\hat{\vD}_{ij}$ is diagonal, for all $1 \le i, j \le b$ as argued above.

As argued above, since $\MM$ satisfies \cref{assump:a1_appendix}, then there exists a matrix ($\vC_1$) that simultaneously diagonalizes all the $\vF(i,j)$'s. Thus, we can always compute some matrix that simultaneously diagonalizes these matrices (i.e., line 2 of \cref{alg:mm_recovery} will always return a valid solution); we discuss how to actually do this below. By definition of simultaneous diagonalization, this matrix (which we set $\hat{\vC}_1$ to) is invertible.

So, $\hat{\vA}_i = \MM_{i1}\hat{\vC}_1^{-1}$ is invertible for all $i$. Thus $\hat{\vC}_j = \hat{\vA}_1^{-1} \MM_{1j}$ is invertible for all $j$ as well. (Note that the equation $\hat{\vC}_j = \hat{\vA}_1^{-1} \MM_{1j}$ holds by construction of $\hat{\vC}_j$ for $j \ge 2$, and by construction of $\hat{\vA}_1$ when $j = 1$.) As $\hat{\vD}_{ij} = \hat{\vA}_i^{-1} \MM_{ij}\hat{\vC}_j^{-1}$ by definition, we thus have that $\MM_{ij} = \hat{\vA}_i \hat{\vD}_{ij} \hat{\vC}_j$ for all $i, j$.

It remains to show that $\hat{\vD}_{ij}$ is diagonal.
\begin{align*}
\hat{\vD}_{ij} &= \hat{\vA}_i^{-1} \MM_{ij}\hat{\vC}_j^{-1} \\
&= (\MM_{i1} \hat{\vC}_1^{-1})^{-1} \MM_{ij} (\hat{\vA}_1^{-1} \MM_{1j})^{-1} \\
&= \hat{\vC}_1 \MM_{i1}^{-1} \MM_{ij} \MM_{1j}^{-1} \hat{\vA}_1  \\
&= \hat{\vC}_1 (\MM_{i1}^{-1} \MM_{ij} \MM_{1j}^{-1} \MM_{11}) \hat{\vC}_1^{-1} \\
&= \hat{\vC}_1 \vF(i, j) \hat{\vC}_1^{-1}%
\end{align*}

But $\hat{\vC}_1 \vF(i, j) \hat{\vC}_1^{-1}$ is diagonal for all $i,j$ by \emph{definition} of $\hat{\vC}_1$ as a matrix that simultaneously diagonalizes the $\vF(i,j)$'s.

As for $\vL_1,\vR,\vL_2$, recall that we can simply set $\vL_1 = \diag(\hat{\vA}_1, \dots, \hat{\vA}_b)$, $\vL_2 = \diag(\hat{\vC}_1, \dots, \hat{\vC}_b)$, and $\vR = \vP_{(b,n)}^\top \lt \begin{array}{cccc} \hat{\vD}_{11} & \hat{\vD}_{12} & \dots & \hat{\vD}_{1b} \\ \hat{\vD}_{21} & \hat{\vD}_{22} & \dots & \hat{\vD}_{2b} \\ \ddots & \ddots & \ddots & \ddots  \\ \hat{\vD}_{b1} & \hat{\vD}_{b2} & \dots & \hat{\vD}_{bb} \end{array}\rt
\vP_{(b,n)}$, and we have $\vM = (\vP_{(b,n)}^\top \vL_1 \vP_{(b,n)}) \vR (\vP_{(b,n)}^\top\vL_2\vP_{(b,n)})$ with $\vL_1, \vL_2 \in \BD\ind{\frac{n}{b}, n}$ and $\vR \in \BD\ind{b,n}$ as argued above. This completes the proof of correctness.

Now, we analyze the runtime. There are $b^2$ matrices $\F(i,j)$ to compute, and computing each one takes $O(\frac{n^3}{b^3})$ time. Once we've found $\hat{\vC}_1$, there are $b$ matrices $\hat{\vA}_i$ to compute, each one taking $O(\frac{n^3}{b^3})$ time, and $b-1$ matrices $\hat{\vC}_j$ (for $j \ge 2$) to compute, each one taking $O(\frac{n^3}{b^3})$ time, and then $b^2$ matrices $\hat{\vD}_{ij}$ to compute, each taking $O(\frac{n^3}{b^3})$ time. (Note that we can compute each of these faster using fast matrix multiplication / inversion; however, it turns out not to matter as the simultaneous diagonalization is the bottleneck.)

Finally, we analyze the simultaneous diagonalization runtime. Simultaneous diagonalization of a set of matrices $\{\vG_1, \dots, \vG_k\}$ is equivalent to finding a mutual eigenbasis for the matrices, since if $\vD_i$ is a diagonal matrix and $\vQ \vG_i \vQ^{-1} = \vD_i$, then the $j^{th}$ column of $\vQ$ is an eigenvector of $\vG_i$ with eigenvalue equal to the $j^{th}$ entry of $\vD_i$.

A simple algorithm for simultaneous diagonalizing a set of matrices, assuming that they are in fact simultaneously diagonalizable (which implies that each matrix is individually diagonalizable), is as follows (e.g. see \cite{Conrad_theminimal, gerstner1993numerical}): first, set $i = 1$ and diagonalize the first matrix $\vG_i = \vG_1$ (i.e., find an eigenbasis), and set $\vQ$ to be the diagonalizing matrix (i.e., the matrix of eigenvectors). So, $\vQ \vG_1 \vQ^{-1}$ is diagonal.
By the assumption that the matrices are in fact simultaneously diagonalizable, $\vQ \vG_j \vQ^{-1}$ will be permuted block diagonal for all $j \ne i$ as well: the size of each block corresponds to the multiplicity of the corresponding eigenvalue of $\vG_1$. (Note that if $\vG_1$'s has unique eigenvalues, then the eigenbasis is unique (up to permutation and nonzero scaling), and thus in this case $\vG_1$ uniquely determines the simultaneously diagonalizing matrix, up to arbitrary permutation and nonzero scaling of the rows. In other words, the block size will be 1 in this case, meaning that $\vQ \vG_j \vQ^{-1}$ will be diagonal for all $j$, and we are done.)

So now, we repeat the following for all $i$ up to $k$. Increment $i$ and compute $\vQ \vG_i \vQ^{-1}$. If it is already diagonal, move on. Otherwise, first permute $\vQ \leftarrow \vP \vQ \vP^\top$ so that it is block diagonal (observe that this maintains the property that $\vQ \vG_j \vQ^{-1}$ is diagonal for all $j < i$, since $\vP \vD \vP^\top$ is diagonal for any permutation $\vP$ and diagonal matrix $\vD$). Then for each block of size $> 1$, compute a matrix that diagonalizes that block; denoting the number of blocks (including size-1 blocks) by $b$, let $\vQ_1', \dots, \vQ_b'$ denote the corresponding diagonalizing transformations, or the scalar 1 when the block is of size 1. Finally set $\vQ' \leftarrow \diag(\vQ_1', \dots, \vQ_b') $ and $\vQ \leftarrow \vQ'^{-1} \vQ \vQ'$. By construction, $\vQ \vG_i \vQ^{-1}$ will now be diagonal; also, $\vQ \vG_j \vQ^{-1}$ is still diagonal for all $j < i$, because any linear combination of a set of eigenvectors of a diagonalizable matrix corresponding to a repeated eigenvalue $\lambda$ is itself an eigenvector of that matrix with eigenvalue $\lambda$.

Thus, once we've processed all $k$ of the $\vG_i$'s, $\vQ$ is a matrix that simultaneously diagonalizes all of them. At each step $i$, we compute diagonalizing transformations for square block matrices whose sizes $s_1, \dots, s_k$ sum to $n$. As eigendecomposition (for a fixed desired precision) takes $O(n^3)$ time for an $n \times n$ matrix, this means the total runtime of step $i$ is $O\lt \sum_{j=1}^{k} s_i^3 \rt \le O(n^3)$. Thus the total runtime of the entire simultaneous diagonalization procedure is $O(kn^3)$, where $k$ is the number of matrices. (Note that iterative methods for simultaneous diagonalization also exist \citep{gerstner1993numerical,akema2020approximate} and could be used to speed up this step in practice.)

Applying this to our problem, we have $b^2$ matrices to simultaneously diagonalize, each of size $\frac{n}{b} \times \frac{n}{b}$. This leads to a total runtime of $O\lt b^2 \cdot (\frac{n}{b})^3\rt = O\lt \frac{n^3}{b}\rt$ for the entire simultaneous diagonalization procedure, and thus the runtime of \cref{alg:mm_recovery} is also $O\lt \frac{n^3}{b}\rt$, as desired.

(Note: As can be seen from the above analysis, we don't actually need $\vM$ itself to be invertible---we simply need all its blocks $\MM_{ij}$ to be, so that all the $\vA_i$'s and $\vC_j$'s are, which is a weaker assumption that invertibility of $\vM$ given that we already assumed the $\vD_{ij}$'s are invertible due to the nonzero assumption on the blocks of $\vR$.)

\end{proof}

\section{Experiment Details}
\label{sec:experiment_details}

\subsection{Model Configurations and Hyperparameters}

We summarize the details required to replicate our experiments below.

\subsubsection{Image Classification}

\textbf{Baseline Model:} For dense models, we use standard implementations of
ViT~\citep{dosovitskiy2020image}, MLP-Mixer{tolstikhin2021mlp} from the
\texttt{timm} library and from the T2T-ViT codebase~\citep{yuan2021tokens}.

The Monarch version of these models simply swap out the dense weight matrices in the attention blocks (projection matrices) and in the FFN block (linear layers) with Monarch matrices.
We set the number of blocks in the block-diagonal matrices to 4.
We also reduce the amount of regularization (stochastic depth) as our Monarch models are smaller than the dense models.

We adopt the hyperparameters (optimizer, learning rate, learning rate
scheduler) from~\citet{yuan2021tokens}.
Details are in~\cref{table:imagenet_hparams}.

We measure the wall-clock training time on V100 GPUs.

\begin{table}[!htbp]
 \caption{Configuration of the ImageNet experiment}   
\centering
\resizebox{0.8\linewidth}{!}{
\noindent\begin{tabular}{@{}c||ccccccc@{}}
  \specialrule{.15em}{.05em}{.05em}
Model&\multicolumn{1}{c}{Optimizer}&\multicolumn{1}{c}{Weight Decay}&\multicolumn{1}{c}{Learning Rate}&\multicolumn{1}{c}{Drop Path}&\multicolumn{1}{c}{Warmup/Epoch}\\
  \specialrule{.15em}{.05em}{.05em}
ViT-Small& AdamW & 0.05 & 0.001 & 0.1& 5/300 \\
Monarch-ViT-Small& AdamW & 0.05 & 0.001 &0& 5/300 \\
ViT-Base& AdamW & 0.05 & 0.001 &0.1& 5/300 \\
Monarch-ViT-Base& AdamW & 0.05 & 0.001 &0& 5/300 \\
  \specialrule{.15em}{.05em}{.05em}
Mixer-Small &AdamW& 0.1 &0.001&0.1& 5/300 \\
Monarch-Mixer-Small &AdamW&0.1 &0.001& 0 & 5/300 \\
Mixer-Base &AdamW& 0.1 &0.001&0.1& 5/300 \\
Monarch-Mixer-Base &AdamW &0.1 &0.001& 0 & 5/300 \\
  \specialrule{.15em}{.05em}{.05em}
\end{tabular}
}
\label{table:imagenet_hparams}
\end{table}

We follow the naming convention in the Vision Transformer paper and MLP-Mixer paper. In particular, ViT-S and ViT-B refers to the small and base ViT models respectively, and 16 refers to the patch size of 16x16. The MLP-Mixer models follow the same convention.

\subsubsection{Language Modeling}
For dense models, we use standard implementations of
GPT-2~\citep{radford2019language} from Huggingface \texttt{transformers} library and from Nvidia's Megatron-LM repo. 
We follow the training recipe of the Megatron-LM repo.

The Monarch version of these models simply swap out the dense weight matrices in the attention blocks (projection matrices) and in the FFN block (linear layers) with Monarch matrices.
We set the number of blocks in the block-diagonal matrices to 4.
We also reduce the regularization strength (dropout) as our model is smaller.

We report the hyperparameters used in~\cref{table:wt103} and~\cref{table:owt}.
We use an effective batch size of 512, and use gradient accumulation to fit into available GPU memory.

We measure the wall-clock training time on V100 GPUs.
\begin{table}[!h]
    \vspace{-0.5cm}
\centering
\caption{Configuration of the WikiText-103 experiments}
\resizebox{0.8\linewidth}{!}{
\noindent\begin{tabular}{@{}c||ccccccc@{}}
  \specialrule{.15em}{.05em}{.05em}
Model&\multicolumn{1}{c}{Optimizer}&\multicolumn{1}{c}{Weight Decay}&\multicolumn{1}{c}{Learning Rate}&\multicolumn{1}{c}{Dropout}&\multicolumn{1}{c}{Warmup/Epoch}\\
  \specialrule{.15em}{.05em}{.05em}
GPT-2-small& AdamW & 0.1 & 6e-4 & 0.1& 10/100 \\
Monarch-GPT-2-small& AdamW & 0.1 & 6e-4 & 0.0 & 10/100 \\
GPT-2-medium& AdamW & 0.1 & 1.5e-4 & 0.1& 10/100 \\
Monarch-GPT-2-medium & AdamW & 0.1 & 1.5e-4 & 0.0 & 10/100 \\
  \specialrule{.15em}{.05em}{.05em}
\end{tabular}
}
\label{table:wt103}
\end{table}

\begin{table}[!h]
\vspace{-0.5cm}
\centering
\caption{Configuration of the OpenWebText experiments}
\resizebox{0.8\linewidth}{!}{
\noindent\begin{tabular}{@{}c||ccccccc@{}}
  \specialrule{.15em}{.05em}{.05em}
Model&\multicolumn{1}{c}{Optimizer}&\multicolumn{1}{c}{Weight Decay}&\multicolumn{1}{c}{Learning Rate}&\multicolumn{1}{c}{Dropout}&\multicolumn{1}{c}{Warmup/Total iterations}\\
  \specialrule{.15em}{.05em}{.05em}
GPT-2-Small& AdamW & 0.1 & 6e-4 & 0.1& 4k/400k \\
Monarch-GPT-2-Small & AdamW & 0.1 & 6e-4 & 0.0 & 4k/400k \\
GPT-2-Medium& AdamW & 0.1 & 1.5e-4 & 0.1& 4k/400k \\
Monarch-GPT-2-Medium & AdamW & 0.1 & 1.5e-4 & 0.0 & 4k/400k \\
  \specialrule{.15em}{.05em}{.05em}
\end{tabular}
}
\label{table:owt}
\end{table}

\subsection{Details for PDE Solving}
We adopt the experiment setting and data generation of Navier-Stokes Equation from FNO~\citep{li2020fourier}. It considers the 2-d Navier-Stokes equation for a viscous, incompressible fliud in vorticity form on the unit tortus:
\begin{align}
    \partial_{t} w(x, t) + u(x, t) \cdot \nabla w(x, t) & = v \Delta w(x, t) + f(x), & x \in (0, 1)^2, t \in (0, T] \\
    \nabla w(x, t) & = 0, & x \in (0, 1)^2, t \in (0, T] \\
    w(x, 0) & = w_0(x), & x \in (0, 1)^2 \\
\end{align}
where $u \in C([, T0])$;$H_{per}((0, 1)^2; \mathbb{R}^2))$ for any $r>0$ is the velocity field, $w=\nabla \times u$ is the vorticity, $w_0 \in L^2_{per}((0, 1)^2; \mathbb{R})$ is the initial vorticity, $v \in \mathbb{R_{+}}$ is the viscosity coefficient, and $f \in L_{per}^2((0, 1)^2; \mathbb{R})$ is the forcing function. 
$T$ represents the time interval since it is time-dependent equation. $v$ represents the viscosity. N represents the number of training pairs or data. \cref{table:pde} shows the results for viscosities $v=1e-3, 1e-4, 1e-5$, $T=50, 30, 20$ respectively and use $N=1000$. 

\subsection{Details for GPT-2 Downstream Tasks}
We train Pixelfly-GPT2-small on a larger scale dataset, OpenWebText, and evaluate the downstream quality on zero-shot generation and classification tasks from~\citep{zhao2021calibrate}, achieving comparable and even better performance to the dense model. Specifically, the datasets contains five popular classification tasks: SST2, Trec, CB, Agnews, and Dbpedia. We also adapated the calibrated metric from~\citep{zhao2021calibrate} for evaluation. Results for each individual task are shown in~\cref{table:gpt_finetune_full}. 

\begin{table}[h]
  \small
  \centering
  \vspace{-3mm}
  \caption{\label{table:gpt_finetune_full}The performance (accuracy) of GPT-2-medium trained with Monarch reverse sparsification and with conventional dense training on text classification benchmarks.}
  \setlength{\tabcolsep}{5pt}
  \vspace{1em}
   \resizebox{0.7\linewidth}{!}{
  \begin{tabular}{@{}c||ccccc@{}}
    \specialrule{.15em}{.05em}{.05em}
    Model&\multicolumn{1}{c}{OpenWebText (ppl)}&\multicolumn{1}{c}{Speedup}& \multicolumn{1}{c}{Classification (avg acc)} \\
    \specialrule{.15em}{.05em}{.05em}
    GPT-2m& 68.3 & 37.0 & 10.7 & 52.0 & 26.6\\
    Monarch-GPT-2m& 72 & 38.6 & 12.5 & 47.3 & 23.0 \\
    \specialrule{.15em}{.05em}{.05em}
  \end{tabular}
  }
  \vspace{-3mm}
\end{table}

\subsection{Details for BERT Pretraining}
\label{subsec:bert_details}

We follow the training procedure and hyperparameters of the reference
implementation from Nvidia Deep Learning examples
(\url{https://github.com/NVIDIA/DeepLearningExamples}).
In particular, we use the LAMB optimizer with learning rate 4e-3.
We use as large a minibatch size as possible that still fits in the GPU memory
(A100-40GB), and use gradient accumulation to reach an effective batch size of
64k sequences for phase 1 (maximum sequence length 128) and 32k for phase 2
(maximum sequence legnth 512).
We train is mixed precision (fp16 and fp32).

We use all the optimizations that were in Nvidia's BERT implementation
in MLPerf 1.1:
\begin{enumerate}
  \item Only compute the prediction scores (last layer) for masked tokens as
  the outputs of other tokens are not used to compute the masked language
  modeling loss.
  \item Remove padding tokens and only compute the attention for non-padding
  tokens.
  \item Use a fused CUDA kernel (FMHA) that combines 4 steps into one kernel: computes
  $Q K^T$, take softmax, apply dropout, multiply by $V$, where $Q, K, V$ are the
  query, key, and value respectively.
  \item Fuse matrix multiplication and adding bias into one CUDA kernel in the feed-forward network
  (FFN) layers. The gradient of the bias is also fused with the matrix
  multiplication the backward pass.
  \item Fuse matrix multiplication and adding bias into one CUDA kernel in the
  attention output projection.
  \item Fuse dropout and adding residual in the residual connection at the end
  on the attention and FFN blocks.
\end{enumerate}

We train with DeepSpeed~\citep{rasley2020deepspeed} ZeRO optimizer stage 1 to
shard the optimizer states, thus reducing GPU memory usage and allowing us to
use larger batch sizes.
For the Nvidia MLPerf implementation, we report the speed for both Apex's
automatic mix-precision (AMP) level O2 (as in the original implementation), and
DeepSpeed ZeRO optimizer.

\subsection{Accelerated Multi-coil MRI Reconstruction}
\label{sec:experiment_details_mri}

\subsubsection{Background}
In multi-coil MRI, multiple receiver coils (i.e. sensors) acquire complex-valued measurements in the spatial frequency (a.k.a. \textit{k-space}) domain. These measurements are modulated by the spatially-varying sensitivity maps, which characterize the sensitivity of each coil to the imaging target. In accelerated MRI, scan times are reduced by decreasing the number of samples acquired in k-space. Because the data is sampled below the Nyquist rate, reconstructing the underlying image is an ill-posed problem.

The forward problem for accelerated multi-coil MRI can be written as the matrix equation
\begin{equation*}
    y = \Omega\boldsymbol{F}\boldsymbol{S}x + \epsilon
\end{equation*}
where $\Omega$ is the binary undersampling mask that indexes acquired samples in k-space, $y$ is the vectorized measured signal in k-space, $\boldsymbol{F}$ is the discrete Fourier transform matrix, $\boldsymbol{S}$ is the receiver coil sensitivity maps,  $x$ is the ground-truth signal in image-space, and $\epsilon$ is additive complex Gaussian noise. The acceleration factor is given by $R = \frac{\sum_i^{|N|} \Omega_i}{|\Omega|}$.

\subsubsection{Experimental Details}

\paragraph{Dataset.} We benchmark our method on the SKM-TEA Raw Data Track, which consists of dual-echo 3D MRI scans \citep{desai2021skm}. Scans are accelerated using Poisson Disc undersampling masks distributed with the dataset. During training, Poisson Disc masks are generated, cached, and applied to mask the k-space data to simulate accelerated scans.

\paragraph{Matrix Shape.} Like all matrices, Monarch matrices have an explicit shape constraint, which is a limitation of these matrices for MRI reconstruction tasks. Thus, the SKM-TEA dataset was filtered to include scans of shape $512 \times 512 \times 160$, which is the most frequently occuring scan shape. A total of 3 scans were dropped from the original 155 scans in the dataset. Our method and all baselines were trained on this filtered dataset.

\begin{table}[!ht]
    \vspace{-0.5cm}
\centering
\caption{Baseline configurations of the SKM-TEA MRI reconstruction experiments.}
\resizebox{0.6\linewidth}{!}{
\noindent\begin{tabular}{c||cccccc}
  \specialrule{.15em}{.05em}{.05em}
Model & Params & Optimizer & Weight Decay & Learning Rate & Epoch \\
  \specialrule{.15em}{.05em}{.05em}
SENSE & --- & --- & --- & --- & --- \\
U-Net &  7.8M & Adam & 1e-4 & 1e-3 & 20 \\
mSENSE  & 57.5K & Adam & 1e-4 & 1e-3 & 20 \\
  \specialrule{.15em}{.05em}{.05em}
\end{tabular}
}
\label{table:skmtea-config}
\end{table}

\paragraph{Baselines.} We compare our method to two baselines, SENSE and U-Net. Parameter count and hyperparameters are available in Table \ref{table:skmtea-config}.
\begin{itemize}
    \item \textit{SENSE}: SENSE performs a linear combination of the images acquired on each coil \citep{pruessmann1999sense}. Here, the inverse fsat Fourier transform (IFFT) is applied to the acquired k-space for each coil. The resulting images are combined into a single complex image by weighting each coil image by corresponding coil sensitivity maps. In accelerated MRI, the unsampled frequencies are zero-valued; thus, SENSE produces a \textit{zero-filled image}. Note, SENSE does not require any training.
    \item \textit{U-Net}: U-Net is a popular fully convolutional neural network baseline for MRI reconstruction \citep{ronneberger2015u}. We use the default implementation and hyperparameters used by \citet{desai2021skm} to benchmark the SKM-TEA dataset. In this approach, the SENSE-reconstructed zero-filled image is mapped to SENSE-reconstructed ground truth images.
\end{itemize}

\paragraph{Monarch-SENSE (mSENSE):} We propose a modification to the SENSE method, in which the (IFFT) is parameterized by a factorized Monarch matrix. This matrix is initialized to the IFFT but, unlike SENSE, is learnable. While mSENSE is trainable, it has 137x fewer trainable parameters than U-Net.

\paragraph{Metrics:} We evaluate reconstruction performance using peak signal-to-noise ratio (pSNR) and structural similarity (SSIM) on both echoes (echo1 - E1, echo2 - E2) separately. Both metrics were computed on the 3D volume of each echo.

\paragraph{Extended Results.} We provide sample reconstructions of SENSE, mSENSE, and U-Net in data-limited settings for first (Fig.~\ref{fig:mri-data-limited-echo1}) and second (Fig.~\ref{fig:mri-data-limited-echo2}) echoes. Both SENSE and U-Net reconstructed images have aliasing artifacts. Due to the random Poisson Disc undersampling pattern, these artifacts are incoherent, causing them to manifest as blurring around fine structures and edges. In contrast, mSENSE can recover these structures with higher fidelity. Even in the second echo, which has lower signal-to-noise ratio (SNR) than the first echo, mSENSE does not overblur the image.

\begin{figure}
    \centering
    \includegraphics[width=0.9\linewidth]{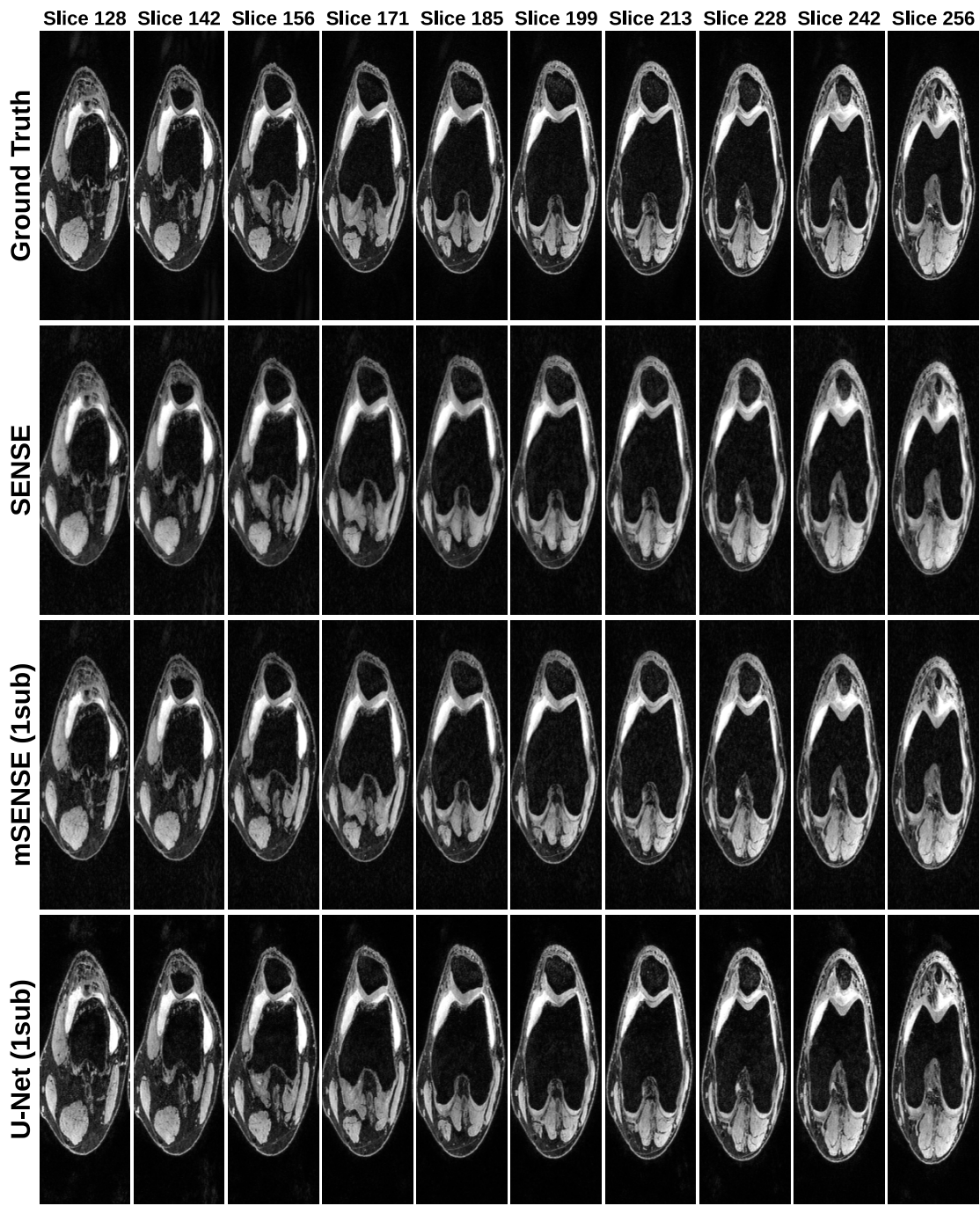}
    \vspace{-1em}
    \caption{Sample reconstructions at 2x acceleration for the first echo in the SKM-TEA dataset using SENSE, Monarch-SENSE (mSENSE), and U-Net. Both mSENSE and U-Net are trained with 1 training scan. SENSE is an untrained method.}
    \label{fig:mri-data-limited-echo1}
\end{figure}

\begin{figure}
    \centering
    \includegraphics[width=6in]{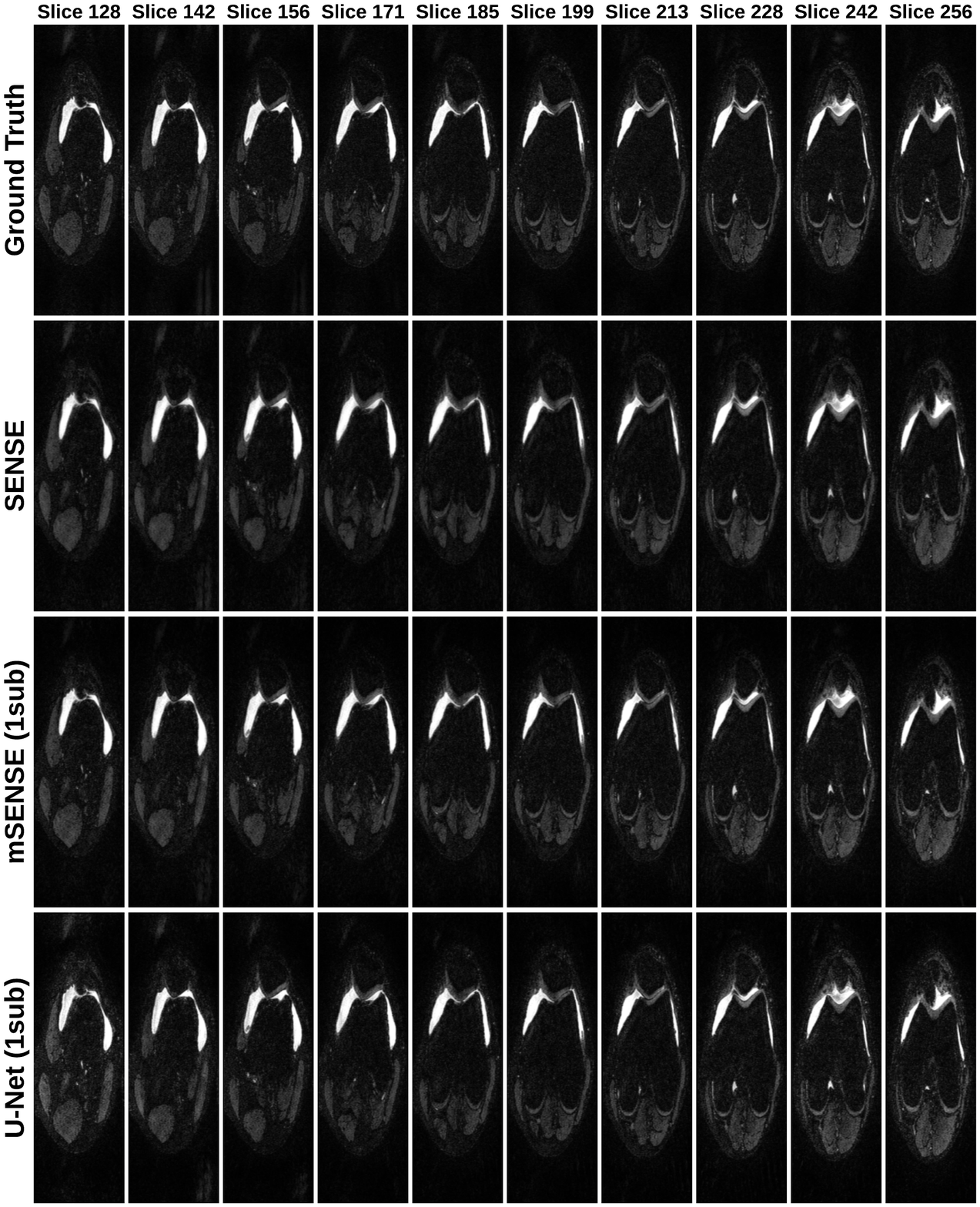}
    \vspace{-1em}
    \caption{Sample reconstructions at 2x acceleration for the second echo in the SKM-TEA dataset using SENSE, Monarch SENSE (mSENSE), and U-Net. Both mSENSE and U-Net are trained with 1 training scan. SENSE is an untrained method.}
    \label{fig:mri-data-limited-echo2}
\end{figure}

\end{document}